\documentclass[sigconf]{aamas} 

\usepackage{amsthm}
\usepackage{bbm}
\usepackage{array, mathtools}
\usepackage{caption}
\usepackage{subcaption}
\usepackage{multirow}
\usepackage{algpseudocode}
\usepackage{algorithm}

\DeclarePairedDelimiterX{\innerp}[1](){\innerpargs{#1}}
\NewDocumentCommand{\innerpargs}{>{\SplitArgument{1}{;}}m}
{\innerpargsaux#1}
\NewDocumentCommand{\innerpargsaux}{mm}
{\IfNoValueTF{#2}{#1}{#1\:\delimsize\Vert\:\mathopen{} #2}}

\usepackage{array}
\newcolumntype{P}[1]{>{\centering\arraybackslash}p{#1}}
\newcolumntype{M}[1]{>{\centering\arraybackslash}m{#1}}

\newtheorem{thm}{Theorem}
\newtheorem{lemma}{Lemma}

\theoremstyle{definition}

\newtheorem{restatedthminner}{Theorem}
\newenvironment{restatedthm}[1]{%
  \renewcommand\therestatedthminner{#1}%
  \restatedthminner
}{\endrestatedthminner}

\usepackage{balance} 

\setcopyright{ifaamas}
\acmConference[AAMAS '24]{Proc.\@ of the 23rd International Conference
on Autonomous Agents and Multiagent Systems (AAMAS 2024)}{May 6 -- 10, 2024}
{Auckland, New Zealand}{N.~Alechina, V.~Dignum, M.~Dastani, J.S.~Sichman (eds.)}
\copyrightyear{2024}
\acmYear{2024}
\acmDOI{}
\acmPrice{}
\acmISBN{}

\acmSubmissionID{41}

\title[Approximate Model-based Shielding in Continuous Environments]{Leveraging Approximate Model-based Shielding for Probabilistic Safety Guarantees in Continuous Environments}

\author{Alexander W. Goodall}
\affiliation{
  \institution{Imperial College London}
  \city{London}
  \country{United Kingdom}}
\email{a.goodall22@imperial.ac.uk}

\author{Francesco Belardinelli}
\affiliation{
  \institution{Imperial College London}
  \city{London}
  \country{United Kingdom}}
\email{francesco.belardinelli@imperial.ac.uk}

\begin{abstract}
Shielding is a popular technique for achieving safe reinforcement learning (RL). However, classical shielding approaches come with quite restrictive assumptions making them difficult to deploy in complex environments, particularly those with continuous state or action spaces. In this paper we extend the more versatile {\em approximate model-based shielding} (AMBS) framework to the continuous setting. In particular we use {\em Safety Gym} as our test-bed, allowing for a more direct comparison of AMBS with popular constrained RL algorithms. We also provide strong probabilistic safety guarantees for the continuous setting. In addition, we propose two novel penalty techniques that directly modify the policy gradient, which empirically provide more stable convergence in our experiments. 
\end{abstract}


\newcommand{\BibTeX}{\rm B\kern-.05em{\sc i\kern-.025em b}\kern-.08em\TeX}

\begin{document}


\pagestyle{fancy}
\fancyhead{}


\maketitle 



\section{Introduction}
\label{sec:introduction}
The construction of optimal controllers and agents is challenging for all but the simplest of environments. {\em Reinforcement learning} (RL) is a principled and convenient framework that allows practitioners to train optimal agents by specifying a reward function to maximise. However, while agents trained with RL can achieve near optimal performance in expectation, they come with no guarantees on worst-case performance. Derived from formal methods, {\em shielding} \cite{alshiekh2018safe} is a popular method for ensuring the safety of RL systems, that comes with strong safety guarantees. In general, shielding is a policy agnostic approach, i.e.~the safety of the system does not rely on the learned RL policy. This is a key advantage over constrained RL approaches \cite{achiam2017constrained, ray2019benchmarking} that rely on the assumption of converge to guarantee safety. 

Shielding is a correct by construction approach. The shield itself prevents learned policies from entering unsafe states defined by some temporal logic formula, typically expressed in {\em Linear Temporal Logic} (LTL) \cite{baier2008principles} or similar.  The shield can be applied preemptively, modifying the action space of the agent to include only those actions verified safe. Alternatively, the shield can be applied in a post-posed fashion, rejecting actions proposed by the agent until a safe action is proposed. However, without additional penalty techniques \cite{alshiekh2018safe, yang2023safe} these rejection based shields can harm convergence to the optimal policy, as the agent is unaware that their actions are being overridden. 

To obtain a shield we must first construct a safety automaton that encodes the given temporal logic formula; a safety game \cite{bloem2015shield} is then constructed and solved to synthesise the shield. This procedure typically requires a priori access to the safety-relevant dynamics of the system. This is a key limitation of shielding, as in many real-world systems the environment is not fully known, or too complex to represent as a simple low-level abstraction. Instead we leverage {\em Approximate Model-based Shielding} (AMBS) \cite{goodallb2023approximate}, which can be applied in more realistic settings, where the safety-relevant dynamics are not known in advance. 

\begin{figure}[t]
    \centering
    \includegraphics[width=.46\textwidth]{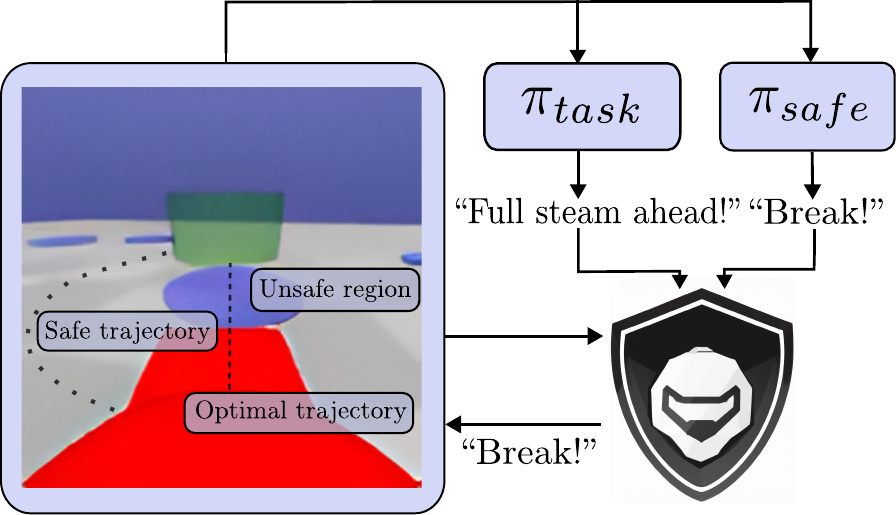}
    \caption{A simple example in Safety Gym \cite{ray2019benchmarking}. The task policy proposes actions along the optimal trajectory. However, this trajectory enters an unsafe region and so the shield overrides these actions with {\em ``Break!''} actions
    proposed by the safe policy. As a result, the safe trajectory is not followed and the two policies continuously fight for control. 
    }
    \label{fig:naiveambs}
\end{figure}

AMBS is a model-based RL algorithm and a general framework for shielding learned RL policies by simulating possible futures in the latent space of a learned dynamics model or {\em world model} \cite{ha2018world}. AMBS is an approximate method that relies on learnt components and Monte Carlo sampling, as such it cannot provide the same strong safety guarantees that classical shielding methods can. That being said, strong probabilistic guarantees have been developed for AMBS, particularly in the tabular case. In our implementation of AMBS we leverage DreamerV3 \cite{hafner2023mastering} as the stand-in dynamics model, used for policy optimisation and shielding.

Na\"ively applying AMBS to Safety Gym \cite{ray2019benchmarking}, results in a high-variance return distribution, see Appendix \ref{sec:additionalresults}. We hypothesise that this phenomena can be explained by the reward policy fighting with the shield to gain control of the system, see Fig.~\ref{fig:naiveambs}. We alleviate this problem by providing the underlying (unshielded) policy with some safety information by using a simple penalty critic (abbrv.~PENL). We also introduce two more sophisticated penalty methods, the first based on {\em Probabilistic Logic Shielding} \cite{yang2023safe} (abbrv.~PLPG), and the second loosely inspired by counter-examples (abbrv.~COPT), a familiar concept from model checking \cite{baier2008principles} and verification \cite{10.1007/10722167_15}. There has been some success in leveraging counter-examples for safe RL \cite{gangopadhyay2023counterexample, ji2023probabilistic}. However, our method does differ from these previous approaches. 

\paragraph{Contributions} We summarise our contributions: (1) we extend and apply AMBS \cite{goodallb2023approximate} to the continuous setting, specifically we use Safety Gym \cite{ray2019benchmarking} to obtain a meaningful comparison with other model-based and safety-aware algorithms. (2) We subtly build on the probabilistic safety guarantees of AMBS, by establishing the same sample complexity bounds for continuous state and action spaces. 
(3) We demonstrate that extending AMBS with safety information is crucial for stable convergence to a safe policy. In addition, we introduce two novel penalty methods that empirically improve the stability of the learned policy in the later stages of training. (4) In our experiments we demonstrate that our extended version of AMBS dramatically reduces the total number of safety-violations, compared to other safety-aware RL algorithms, while maintaining good convergence properties and performance w.r.t.~episode return.



\section{Preliminaries}

In this section we describe the problem setup and introduce {\em bounded safety}, an important property that has been formalised in previous work \cite{giacobbe2021shielding, goodalla2023approximate, goodallb2023approximate}. We note that this is also very similar to {\em safety for a finite horizon} in \cite{jansen2020safe}. In addition, we also provide a detailed overview of AMBS \cite{goodallb2023approximate}.

\subsection{Problem Setup}

In Safety Gym \cite{ray2019benchmarking} the agent can learn from low-dimensional lidar observations or high-dimensional visual input. To obtain a meaningful comparison with key prior works \cite{as2022constrained, huang2023safe}, we opt for the high-dimensional visual input setting 
As such, we model the environment as a {\em partially observable Markov decision process} (POMDP) \cite{puterman1990markov}, defined as
a tuple $\mathcal{M} = \langle S, A, p, \iota_{init}, R, \gamma, \Omega, O, AP, L \rangle$, where $S$ and $A$ are the {\em state} and {\em action} spaces respectively; $p : S \times A \times S \rightarrow [0, 1]$ is the {\em transition function}, where $p(s' \mid s, a)$ denotes the probability of transitioning to the new state $s'$ by taking action $a$ from state $s$; $\iota_{init} : S \rightarrow [0, 1]$ is the {\em initial state distribution}; $R : S \times A \rightarrow \mathbb{R}$ is the {\em instantaneous reward function}, $\gamma \in (0, 1]$ is the {\em discount factor}, $\Omega$ is the set of {\em observations}, with $O : S \times \Omega \rightarrow [0, 1]$ as the {\em observation function}. In addition, to the standard POMDP elements, we introduce a set of {\em atomic propositions} (or atoms) $AP$ and a {\em labelling function} $L : S \to 2^{AP}$. This extension of the typical RL framework is common in the model checking literature \cite{baier2008principles}. 

Strictly speaking, POMDPs are not equipped to reason about safety. A typical way of framing the safe RL problem is with the constrained Markov decision process (CMDP) \cite{altman1999constrained}, which introduces the idea of a cost function set and a set of constraints which typically constrain the undiscounted (or discounted) accumulated costs below a given threshold. We however, follow a slightly less conventional setup used in previous work \cite{goodalla2023approximate,goodallb2023approximate}, although the two are closely related.

Specifically, we are given a {\em propositional safety-formula} $\Psi$ that encodes the safety constraints of the environment. To provide more context we present the following simple example from \cite{goodallb2023approximate}. For a simple self-driving car environment we might consider the following safety-formula,
\begin{equation*}
    \Psi = \neg collision \land (red\_light \Rightarrow stop)
\end{equation*}
which intuitively says ``don't have a collision and stop when there is a red light'', for the set of atoms $AP = \{ collision, red\_light, stop\}$. At each discrete timestep $t$, the agent receives an observation $o_t$, reward $r_t$ and a set of labels $L(s_t)$, where $L(s_t) \in 2^{AP}$ denotes the set of atoms that hold in the unobservable current state $s_t$, see Fig.~\ref{fig:pomdplabels}. From the set of labels $L(s_t)$ the agent can then determine whether the current state $s_t$ satisfies the safety-formula $\Psi$ by applying the following satisfaction relation, 
\begin{equation*}
    \begin{array}{@{}r@{{}\mathrel{}}c@{\mathrel{}{}}l@{}}
    s \models a & \text{iff} & a \in L(s)\\
    s \models \neg \Psi & \text{iff} & s \not \models \Psi\\
    s \models \Psi_1 \land \Psi_2 & \text{iff} & s \models \Psi_1 \text{ and } s \models \Psi_2
    \end{array}
\end{equation*}
where $a \in AP$ is an atom, negation ($\neg$) and conjunction ($\land$) are the familiar logical operators from propositional logic.
\begin{figure}[t]
    \centering
    \includegraphics[width=.47\textwidth]{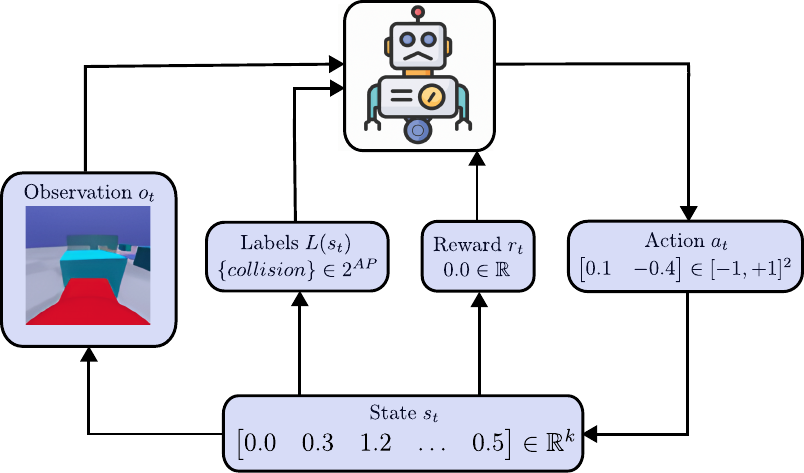}
    \caption{POMDP with Labels. 
    }
    \label{fig:pomdplabels}
\end{figure}
In this setup the goal is to find the optimal policy w.r.t.~reward, that is, $\pi^* = \arg\max_{\pi} \mathbb{E}[\sum_{t=1}^{\infty} \gamma^{t-1} \cdot r_t]$, while minimising the total number of violations of the safety-formula $\Psi$ during training and deployment. However, this may be infeasible, for example, if the optimal policy $\pi^*$ is not `safe' and frequently violates $\Psi$. 

Instead we have a preference for policies that achieve the minimum number of safety-violations over those that achieve high reward. In the CMDP framework, this goal can be interpreted as constraining all accumulated costs to zero (or the corresponding minimal value of the environment), and obtaining the optimal policy from the corresponding feasible set $\Pi_{C}$, i.e., $\pi^* = \arg\max_{\pi \in \Pi_{C}} \mathbb{E}[\sum_{t=1}^{\infty} \gamma^{t-1} \cdot r_t]$. 


In the original AMBS paper \cite{goodallb2023approximate} there was an implicit assumption that the globally optimal policy w.r.t.~reward was `safe' (i.e., achieved zero safety-violations), this is certainly not the case for Safety Gym \cite{ray2019benchmarking} where there is a trade-off between reward and constraint satisfaction. Rather than explicitly formalise this new goal as a CMDP, we use the original setup as it will be particularly useful for reasoning about safety in the subsequent sections. 


\subsection{Bounded Safety}

Bounded safety is a policy level property introduced in \cite{giacobbe2021shielding}, which formally states the safety requirements for a learned RL policy to be deemed `safe'. Bounded safety can also be interpreted as a property over sequences of states or state-action pairs. In \cite{goodallb2023approximate} a probabilistic interpretation of bounded safety, {\em $\Delta$-bounded safety}, was introduced as a temporal property defined with {\em Probabilistic Computation Tree Logic} (PCTL) \cite{baier2008principles}. We start by introducing PCTL.

PCTL is a popular specification language for discrete stochastic systems, that is typically used to specify soft reachability and probabilistic safety properties. The syntax of state formulas $\Phi$ and path formulas $\phi$ in PCTL is defined as follows,
\begin{eqnarray*}
    \Phi & ::= & \text{true} \mid a \mid \neg \Phi \mid \Phi \land \Phi \mid \mathbb{P}_{J}(\phi) \\
    \phi & ::= & X \Phi \mid \Phi U \Phi \mid \Phi U^{\leq n} \Phi
\end{eqnarray*}
where once again $a \in AP$ is an atomic proposition;  negation ($\neg$) and conjunction ($\land$) are the familiar logical operators from propositional logic; $J \subset [0, 1]$, $J \neq \emptyset$ is a non-empty subset of the unit interval, and next $X$, until $U$ and bounded until $U^{\leq n}$ are the base temporal operators from CTL \cite{baier2008principles}. A PCTL property is defined by a state formula $\Phi$ and we say that $s \models \Phi$ iff $s$ satisfies the PCTL property of $\Phi$. On the other hand, path formulas $\phi$ state properties over paths or sequences of states $s_0, s_1, s_2, \ldots$, we note that state formula $\Phi$ are related to path formula $\phi$ by probabilistic quantification, e.g.~$\Phi = \mathbb{P}_{\geq 0.99}(\phi)$, a state $s$ satisfies $\Phi$ if the path formula $\phi$ holds with probability at least 0.99 from $s$. For additional details, e.g.~the satisfaction relation, please refer to \cite{baier2008principles}. 

From this definition we also note that the common temporal operators eventually ($\Diamond$) and always ($\square$), and their bounded counterparts ($\Diamond^{\leq n}$ and  $\square^{\leq n}$) can be derived as $\Diamond \Phi ::= \text{true} \, U \Phi$ and $\Box \Phi ::= \neg \Diamond \neg  \Phi$ (resp.~$\Diamond^{\leq n} \Phi ::= \text{true} \,U^{\leq n} \Phi$ and $\Box^{\leq n} \Phi ::= \neg  \Diamond^{\leq n} \neg  \Phi$). 

We are now equipped to formalise {\em bounded safety}. Consider some fixed (stochastic) policy $\pi $ and POMDP (with labels) $\mathcal{M} = (S, A, p, \iota_{init}, R, \Omega, O, AP, L)$. We will see shortly that this policy $\pi$ operates on belief states or state estimates $\hat s$. Nevertheless, $\pi$ and $\mathcal{M}$ define a transition system $\mathcal{T} : S \times S \to [0, 1]$, where for all $s$, $\int_{s' \in S} \mathcal{T}(s' \mid s) ds' = 1$. A {\em trace} (or path) of $\mathcal{T}$ denoted $\tau$ is a sequence of states $s_0 \to s_1 \to s_2 \to \ldots \to s_n \to \ldots$. A finite trace is an $n$-prefix of a trace $\tau$, with $n < \infty$, we denote as $\tau[i]$ the $i^{\text{th}}$ state of the trace $\tau$. In our setup, a trace $\tau$ satisfies bounded safety if all of its states $s$ satisfy the propositional safety-formula $\Psi$, or formally, 
\begin{eqnarray*}
    \tau \models \square^{\leq n} \Psi & \text{iff} & \ \forall i \; 0 \leq i \leq n, \tau[i] \models \Psi
    \label{eq:boundedsafety}
\end{eqnarray*}
for some bounded look-ahead parameter $n$. Now we define the state level property {\em $\Delta$-bounded safety}. In PCTL a given state $s \in S$ satisfies $\Delta$-bounded safety, or formally, $s \models \mathbb{P}_{\geq 1 - \Delta}(\square^{\leq n} \Psi )$, iff, 
\begin{equation}
   \mu_s(\{\tau \mid \tau[0] = s, \forall i \; 0 \leq i \leq n, \tau[i] \models \Psi\}) \in [1- \Delta, 1] \label{eq:deltaboundedsafety}
\end{equation}
where $\mu_s$ is a is a well-defined probability measure on the set of traces from $s$, induced by the transition probabilities of $\mathcal{T}$, see \cite{baier2008principles} for details. In words, a given state $s \in S$ satisfies $\Delta$-bounded safety iff the proportion (or probability) of satisfying traces from $s$ is greater than $1 - \Delta$. For the rest of this paper we will use $\mu_{s\models \phi}$ to denote the measure $ \mu_s(\{\tau \mid \tau[0] = s, \forall i \;0 \leq i \leq n, \tau[i] \models \Psi\})$.

In principle, the parameter $\Delta$ is used to trade-off safety and exploration, with greater $\Delta$ corresponding to more risky behaviour and vice versa. If we consider $\Delta=0$ (100\% safety) we quickly arrive at the policy level property introduced in \cite{giacobbe2021shielding}. However, checking this level of safety may require the enumeration of all paths of length $n$, which can be quite computationally demanding. The advantage of considering $\Delta>0$, is that we can estimate the measure $\mu_{s\models \phi}$ by Monte Carlo sampling, and check $\Delta$-bounded safety with high probability.

\subsection{Approximate Model-based Shielding}

Inspired by {\em latent shielding} \cite{he2021androids},  Approximate Model-based Shielding (AMBS) \cite{goodallb2023approximate} is a general purpose framework for verifying the safety of learned RL policies in the latent space of a world model \cite{ha2018world}. The core idea of AMBS is to remove the requirement of access to the safety-relevant dynamics of the environment, by instead learning an approximate dynamics model. We can then shield the learned policy at each timestep, by sampling traces from the dynamics model and estimating the probability of a safety-violation in the near future. If this probability is too high (i.e.~above a predefined threshold) then we can override the action proposed by the agent and sample an action from the {\em backup policy} instead. 

We leverage DreamerV3 \cite{hafner2023mastering} as the stand-in dynamics model for AMBS, as it is a particularly flexible state-of-the-art single-GPU model-based RL algorithm. DreamerV3 is based on the {\em Recurrent State Space Model} (RSSM) \cite{hafner2019learning}, a special type of sequential {\em Variational Auto-encoder} (VAE) \cite{kingma2013auto}. We summarise the key model components as follows:
\begin{itemize}
    \item Sequential model: $h_t = f_{\theta}(h_{t-1}, z_{t-1},a_{t-1})$
    \item Image / observation encoder: $z_t \sim q_{\theta}(z_t \mid o_t, h_t)$
    \item Transition predictor (prior): $\hat z_t \sim p_{\theta}(\hat z_t \mid h_t)$
    \item Image/observation decoder $\hat o_t \sim p_{\theta}(\hat o_t \mid h_t, z_t)$
    \item Reward predictor: $\hat r_t \sim p_{\theta}(\hat r_t \mid h_t,z_t )$
    \item Episode termination predictor: $\hat \gamma_t \sim p_{\theta}(\hat \gamma_t \mid h_t, z_t)$
    \item Cost predictor: $\hat c_t \sim p_{\theta}(\hat c_t \mid h_t, z_t)$
\end{itemize}
The RSSM learns a compact latent representation of the environment state by minimising the reconstruction loss between observations $o_t$ and reconstructed targets $\hat o_t$, output by the image / observation decoder. The recurrent state denoted $h_t$ is a function of the previous recurrent state $h_{t-1}$, stochastic latents $z_{t-1}$ and action $a_{t-1}$. The transition predictor uses the recurrent state to predict prior samples $\hat z_t$ which correspond to the stochastic latents. $z_t$ denotes posterior samples that are obtained when the current observation $o_t$ is available as input to the model. By minimising the divergence between the prior and posterior samples (i.e.~$\hat z_t$ and $z_t$ resp.), the model should have the capability to accurately predict future latent states for long horizons. Where $\hat s_t = ( z_t, h_t )$ is the latent state of the model. 

In addition to the standard DreamerV3 components \cite{hafner2023mastering}, AMBS utilises a cost predictor $\hat c_t \sim p_{\theta}(\hat c_t \mid h_t, z_t)$, to provide TD-$\lambda$ targets for the backup policy, which is trained specifically to minimise costs, rather than (maximise) rewards. The cost predictor along with an optional safety critic, are used in the shielding procedure to check whether traces sampled from the world model satisfy bounded safety. We make the distinction here between the task policy $\pi_{\text{task}}$ trained to maximise rewards and the backup policy $\pi_{\text{safe}}$ which is trained to minimise costs. We note that in all but the simplest scenarios the safe policy cannot be completely hand-engineered and must be learnt.

To obtain a shielded policy, AMBS checks that $\Delta$-bounded safety holds from the current state $s_t$, i.e., $s_t \models \mathbb{P}_{\geq 1 -  \Delta}(\square^{\leq n} \Psi )$, under the transition probabilities induced by $\mathcal{T}: S \times S \to [0, 1]$ (where $\mathcal{T}$ is the transition system obtained by fixing the task policy $\pi_{\text{task}}$ in the POMDP $\mathcal{M}$). More specifically, we estimate the measure $\mu_{s_t \models \phi}$ and check if $\mu_{s_t \models \phi} \in [1 - \Delta, 1]$. The shielded policy is then obtained as follows: if $\Delta$-bounded safety holds from the current state then play the action proposed by the task policy $\pi_{\text{task}}$, otherwise sample an action $a \sim \pi_{\text{safe}}$ from the backup policy. In the first case we can be sure to remain safe with probability $\geq 1 - \Delta$, in the second case we can be sure to remain safe provided that the backup policy $\pi_{\text{safe}}$ has converged. 

In practice, we cannot assume access to the true transition system $\mathcal{T}$, rather we obtain an approximation $\widehat{\mathcal{T}} \approx \mathcal{T}$ by rolling out the task policy $\pi_{\text{task}}$ in the latent space of the learned world model $p_{\theta}$. More formally, by fixing $\pi_{\text{task}}$ and $p_{\theta}$, we obtain an approximate transition system $\widehat{\mathcal{T}}: S \times S \to [0, 1]$, that operates on the same set of abstract states $S$. Under reasonable assumptions, and by sampling sufficiently many traces from $\widehat{\mathcal{T}}$, AMBS provides a straightforward method for obtaining an $\epsilon$-approximate estimate of the measure $\mu_{s_t \models \phi}$. We provide the corresponding sample complexity bound for this procedure in Sec.~\ref{sec:safetyguarantees}. 

To check that a given trace $\tau$ satisfies bounded safety we use the cost predictor to compute the discounted cost of the trace, we let $\textrm{cost}(\tau) = \sum^n_{t=1} \gamma^{t-1} \cdot \hat c_t $ denote the cost of a trace. Safety critics, which estimate the cost value function $\mathbb{E}_{\pi_{\text{task}}}[ \sum^{\infty}_{t=1} \gamma^{t-1} \cdot c_t ]$ of the task policy, can be used to bootstrap the cost estimates for longer-horizon shielding capabilities. If the cost of the trace 
is $< \gamma^{n-1} \cdot C$, where $c_t = 0 \text{ if } s_t \models \Psi \text{ and } c_t = C \text{ otherwise}$, then we say that $\tau$ satisfies bounded safety (i.e., $\tau \models \square^{\leq n} \Psi$). By taking several samples and computing the proportion of satisfying traces we obtain an estimate for $\mu_{s_t \models \phi}$. Since AMBS only guarantees an $\epsilon$-approximate estimate measure $\mu_{s_t \models \phi}$, we typically check that the estimate $\tilde \mu_{s_t \models \phi}$ lies in the interval $[1 - \Delta + \epsilon, 1]$, which guarantees that $\Delta$-bounded safety is satisfied (with high probability). The full learning algorithm and shielding procedure are detailed in Appendix \ref{sec:algorithms} and we refer the reader to \cite{goodallb2023approximate} for additional justification and proof details. 


\section{Safety Guarantees}
\label{sec:safetyguarantees}

In this section we establish probabilistic safety guarantees for AMBS \cite{goodallb2023approximate} in the continuous setting. We present Theorem 1 from \cite{goodallb2023approximate}, which develops a sample complexity bound for the true transition system with full state observability, the proof in the continuous case is immediate. For the approximate transition system we establish a similar sample complexity bound. The proof of this is a simple extension of the original theorem from \cite{goodallb2023approximate}, that leverages Pinsker's inequality \cite{tsybakov2004introduction}. We conclude this section by considering the partially observable case and we discuss when the bounds for the fully observable case may become useful. Complete proofs are provided in Appendix \ref{sec:proofs}. 

\subsection{Full Observability}
\label{sec:fullobservability}
We first consider of Markov decision processes (MDPs), i.e.~full observability. An MDP is a tuple $\mathcal{M} = \langle S, A, p, \iota_{init}, R, \gamma, AP, L \rangle$. We can think of an MDP as a special case of POMDP where the observation set $\Omega = S$ and the observation function $O : \Omega \times S \to [0,1]$ is the identity. In a similar way as before, consider some fixed (stochastic) policy $\pi : S \times A \to [0, 1]$ and MDP (with labels) $\mathcal{M} = (S, A, p, \iota_{init}, R, AP, L)$. Together $\pi$ and $\mathcal{M}$ define a transition system $\mathcal{T} : S \times S \to [0, 1]$, where for all $s$, $\int_{s' \in S} \mathcal{T}(s' \mid s) ds' = 1$. We restate the following result from \cite{goodallb2023approximate}.
\begin{thm}
Let $\epsilon > 0$, $\delta > 0$, $s \in S$ be given. With access to the true transition system $\mathcal{T}$, with probability $1 - \delta$ we can obtain an $\epsilon$-approximate estimate of the measure $\mu_{s\models\phi}$, by sampling $m$ traces $\tau \sim \mathcal{T}$, provided that,
\begin{equation}
    m \geq \frac{1}{2\epsilon^2} \log\left(\frac{2}{\delta}\right) \label{eq:boundonm}
\end{equation}
\label{prop:boundonm}
\end{thm}

Theorem \ref{prop:boundonm} outlines the number of samples required to estimate the measure $\mu_{s\models\phi}$ up to $\epsilon$-error, with high probability ($1 - \delta$). However this result assumes we can sample directly from the true transition system $\mathcal{T}$. Instead we consider the more realistic case, where we no longer have access to $\mathcal{T}$, but rather an approximate transition system $\widehat{\mathcal{T}}$. Again more formally, we consider some fixed (stochastic) policy $\pi : S \times A \to [0, 1]$ and a learned approximation $p_{\theta}$ of the MDP dynamics. Together $\pi$ and  $p_{\theta}$ define the approximate transition system $\widehat{\mathcal{T}} : S \times S \to [0, 1]$, where for all $s$, $\int_{s' \in S} \widehat{\mathcal{T}}(s' \mid s) ds' = 1$. We state the following new result.

\begin{thm}
Suppose that for all $s \in S$, the Kullback-Leibler (KL) divergence\footnote{The KL divergence between two distributions $p$ and $q$ is defined as $D_{KL}\innerp*{p ; q} = \int_{\mathcal{X}} p(x) \log\frac{p(x)}{q(x)} dx$, for some (possibly uncountable infinite) set $\mathcal{X}$.} between the distributions $\mathcal{T}(\cdot \mid s)$
and $\widehat{\mathcal{T}}(\cdot \mid s)$ is upper-bounded by some $\alpha \leq \epsilon^2/(2 n^2)$. That is,
\begin{equation}
	D_{KL}\innerp*{\mathcal{T}(\cdot \mid s) ; \widehat{\mathcal{T}}(\cdot \mid s) } \leq \alpha \; \forall s \in S 
 \label{eq:tvdistance}
\end{equation}
Now fix an $s \in S$ and let $\epsilon > 0$, $\delta> 0$ be given. With probability $1 - \delta$ we can obtain an $\epsilon$-approximate estimate of the measure $\mu_{s\models\phi}$, by sampling $m$ traces $\tau \sim \widehat{\mathcal{T}}$, provided that,
\begin{equation}
    m \geq \frac{2}{\epsilon^2} \log\left(\frac{2}{\delta}\right)
    \label{eq:boundonmbig}
\end{equation}
\label{prop:kl}
\end{thm}
\begin{proof}
The proof follows from Theorem \ref{prop:boundonm}, Pinsker's inequality \cite{tsybakov2004introduction} and the {\em simulation lemma} \cite{kearns2002near} adapted to our purposes, see Appendix \ref{sec:proofs} for details. 
\end{proof}

We use KL divergence here rather than variation distance, to obtain a more general result, that easily extends to many common continuous settings. Some examples include, Gaussian dynamics centred at the current state \cite{brunskill2009provably}, linear MDPs, and kernelised non-linear regulators (KNRs) \cite{kakade2020information, mania2020active}, which capture popular models in the robotics community, such as linear quadratic regulators (LQR) \cite{abbasi2011regret}, piece-wise linear systems, non-linear higher-order polynomial systems and Gaussian processes (GP) \cite{deisenroth2011pilco}. For isotropic Gaussian dynamics the KL divergence reduces to the $l_2$ (squared) loss, similarly for linear MDPs and more general KNRs the KL divergence has a convenient form.

Theorem \ref{prop:kl} provides us with a sample complexity bound for estimating $\mu_{s \models \phi}$ up to $\epsilon$ error, under the assumption that the KL divergence between the {\em true } transition system $\mathcal{T}$ and approximate transition system $\widehat{\mathcal{T}}$ is bounded. In general, this assumption is quite strong and may only be met when the entire state space is sufficiently explored. However, in reality we may only require the KL divergence to be bounded for the safety-relevant subset of the state space. Moreover, for a number of continuous settings, sample complexity results have been developed \cite{song2021pc, kakade2020information, brunskill2009provably, abbasi2011regret}, so that we known when Eq.~\ref{eq:tvdistance} is satisfied with high probability.


\subsection{Partial Observability}

Without additional assumptions, sample complexity bounds for POMDPs are difficult to obtain \cite{lee2023learning}. A common way to reason about the underlying states of a POMDP is with {\em belief states}, which are defined as a filtering distribution $p(s_t \mid o_{\leq t}, a_{\leq t})$, where $(o_{\leq t}, a_{\leq t})$ is the history of past observations and actions. However, rather than maintaining an explicit distribution over the states, we consider a latent representation of the history $b_t = f(o_{\leq t}, a_{\leq t})$. The idea is that this latent representation captures the important statistics of the filtering distribution, so that $p(s_t \mid o_{\leq t}, a_{\leq t}) \approx p(s_t \mid b_t)$. 

To develop an algorithm for this setting we present a theorem adapted from \cite{goodallb2023approximate}, which shows that -- the divergence between belief states in $\mathcal{T}$ and $\widehat{\mathcal{T}}$ is an upper bound on the divergence between the underlying states in $\mathcal{T}$ and $\widehat{\mathcal{T}}$. This bound holds for generic $f$-divergence measures, which includes KL divergence and variation distance.

\begin{thm} \label{prop:pomdp} Let $b_t$ be a latent representation (belief state) such that $p(s_t \mid o_{t\leq t}, a_{\leq t}) = p(s_t \mid b_t)$. 
Let the fixed policy $\pi(\cdot \mid b_t)$ be a general probability distribution conditional on belief states $b_t$. Let $f$ be a generic $f$-divergence measure (e.g., KL divergence). Then the following holds:
\begin{equation}
    D_{f}(\mathcal{T}(s' \mid b ), \widehat{\mathcal{T}}(s' \mid b )) \leq D_{f}(\mathcal{T}(b' \mid b ), \widehat{\mathcal{T}}(b' \mid b ))
\end{equation}
where $\mathcal{T}$ and $\widehat{\mathcal{T}}$ are the true and approximate transition system respectively, defined now over both states $s$ and belief states $b$.
\end{thm}

In DreamerV3 \cite{hafner2023mastering} the learned policy operates on the posterior belief state representation (i.e.,~$\hat s_t = (z_t, h_t)$), which is computed with the current observation $o_t$ as input to the model. We can think of the transition system $\mathcal{T}$ as the recurrent model being used with access to $o_t$ at each timestep. Alternatively, we can think of the transition system $\widehat{\mathcal{T}}$ as the recurrent model being used autoregressively, where the belief states are sampled from the prior transition predictor (i.e.,~$\hat z_t \sim p_{\theta}(\hat z_t \mid h_t)$). In the DreamerV3 objective function the KL divergence between the prior and posterior belief states is minimised, which then minimises an upper bound on the KL divergence between $\mathcal{T}$ and $\widehat{\mathcal{T}}$ on the underlying states. And so Theorem \ref{prop:pomdp} provides a direct way to extend the sample complexity bounds introduced in Sec.~\ref{sec:fullobservability} to POMDPs.

By leveraging world models \cite{hafner2019learning, hafner2023mastering} we hope to learn a suitable latent space that captures the sufficient statistics of the observation and action history. We can then pick suitable values for the hyperparameters $\Delta$, $\epsilon$, $\delta$ and $n$, that reflect the desired level of safety we want to attain, and we can pick an $m$ (number of traces sampled) that satisfies the bound in Eq.~\ref{eq:boundonmbig}.




\section{Penalty Techniques}
\label{sec:penalty}

In this section we present three penalty techniques that directly modify the policy gradient, to provide the task policy with safety information, with the idea that it will converge to a policy that balances both reward and safety. We first introduce a simple penalty critic (abbrv.~PENL), which acts as a baseline for the two more sophisticated penalty techniques, PLPG and COPT, that are presented later on in Sec.~\ref{sec:plpg} and Sec.~\ref{sec:copt} respectively.

\subsection{Penalty Critic (PENL)}
\label{sec:penaltycritic}

In this subsection we introduce a simple penalty critic into the policy gradient for the task policy $\pi_{\text{task}}$. We first note the form of  the standard REINFORCE policy gradient \cite{sutton2018reinforcement},
\begin{equation}
    \nabla \mathcal{J} = \mathbb{E}_{\pi_{\text{task}}} \Big[ \sum^H_{t=1} G_t \cdot \nabla \log \pi_{\text{task}}(a_t \mid s_t)\Big]
    \label{eq:policygradient}
\end{equation}
where $G_t$ is the full Monte Carlo return, which can be replaced by either the advantages $A_t = Q(s_t, a_t) - V(s_t)$, or TD-$\lambda$ targets $R^{\lambda}_t$, see \cite{hafner2023mastering} for details.

The penalty critic estimates the discounted accumulated cost from a given state, taking the expectation under the task policy. Specifically, we have, $V^{\pi_{\text{task}}}_C(s) = \mathbb{E}_{\pi_{\text{task}}}[\sum^{\infty}_{t=1} \gamma^{t-1} \cdot c_t \mid s_0 = s]$. We propose the following new form for the policy gradient,
\begin{equation}
    \nabla \mathcal{J} = \mathbb{E}_{\pi_{\text{task}}} \Big[ \sum^H_{t=1} (G_t - \alpha \cdot G^C_t) \cdot \nabla \log \pi_{\text{task}}(a_t \mid s_t)\Big]
    \label{eq:penaltypolicygradient}
\end{equation}
where $G^C_t$ is the cost return and $\alpha$ is a hyperparameter that weights the cost returns $G^C_t$ against the standard returns $G_t$.

By modifying the policy gradient in this simple way, the task policy is optimised to maximise reward and minimise cost, so that we may avoid the issue presented earlier in Fig.~\ref{fig:naiveambs} (Sec.~\ref{sec:introduction}).

We note that this is a very straightforward way of using the penalty (or cost) critic and there are certainly more sophisticated approaches. As a baseline for our experiments, we consider a version of DreamerV3 \cite{hafner2023mastering} that implements the Augmented Lagrangian penalty framework \cite{wright2006numerical}, which adpatively modifies the coefficient of the penalty critic, to obtain a smooth approximation of the corresponding constrained optimisation problem. We provide a description of this framework in Appendix \ref{sec:augmentedlagrangian} and for additional implementation details we refer the reader to \cite{as2022constrained}.

\subsection{Probabilistic Logic Policy Gradient (PLPG)}
\label{sec:plpg}

The {\em Probabilistic Logic Policy Gradient} (PLPG) \cite{lee2023learning} connects safety-aware policy optimisation with probabilistic safety semantics. In particular, the PLPG directly optimises a shielded policy which re-normalises the probabilities of the actions of a given base policy; by increasing the probability of actions that are safer on average and decreasing those that aren't. The key motivation of this probabilistic shielding framework, is that it comes with the same convergence guarantees as standard policy gradient methods. In contrast to classical rejection based shielding methods, where these guarantees are unobtainable.

The original form of the PLPG \cite{yang2023safe} only considers one step probabilistic safety semantics, and so we adapt it to our purposes as follows. We use bounded safety semantics to specify the following re-normalising coefficient of the returns $G_t$,
\begin{equation}
    \frac{\mathbb{P}_{\pi_{\text{safe}}}\left[s_t \models \square^{\leq H} \Psi \mid s_t, a_t \right]}{\mathbb{P}_{\pi_{\text{safe}}}\left[s_t \models \square^{\leq H} \Psi \mid s_t\right]}
\end{equation}

In words, the numerator corresponds to the probability (under the backup policy, i.e., ~$\pi_{\text{safe}}$) of satisfying bounded safety by taking action $a_t$ from $s_t$ and the divisor is the marginal probability under actions selected from the backup policy. The intuition is that for critical actions that are unsafe this term will be small, and so the gradient of these actions will be scaled to zero. In addition, PLPG modifies the policy gradient with a log probability penalty term, $- \alpha \log \mathbb{P}_{\pi_{\text{task}}}[s_t \models \square^{\leq H} \Psi \mid s_t]$, which penalises unsafe states. 

Using the control as inference framework (i.e.,~$p( O_t = 1) = \exp(R(s_t, a_t))$) \cite{levine2018reinforcement}, we can derive a convenient form for the policy gradient,
\begin{equation}
	\nabla \mathcal{J} = \mathbb{E}_{\pi_{\text{task}}}\Big[  \sum^{\infty}_{t=0} \big( sg(\exp(\delta^{\text{safe}}_{t})) \cdot G_t - \alpha \cdot G_t^C \cdot \big) \nabla\log \pi_{\text{task}}(a_t \: \mid \: s_t)  \Big]
 \label{eq:plpg}
\end{equation}
where $\delta^{\text{safe}}_{t} = c(s_t, a_t) + V^{\pi_{\text{safe}}}_C(s_{t+1}) - V^{\pi_{\text{safe}}}_C(s_t)$, is the (undiscounted) temporal difference (TD) of the backup policy cost return, $sg(\cdot)$ is the stop gradient operator and $G^C_t$ is once again the cost return of the task policy. We provide a full derivation of this policy gradient in Appendix \ref{sec:derivations}.

\subsection{Counter-example Guided Policy Optimisation (COPT)}
\label{sec:copt}

In the context of model checking \cite{baier2008principles}, a counter-example is defined as a trace of the system that violates the desired property being checked. We introduce a novel penalty method based on this idea. In DreamerV3 \cite{hafner2019learning}, during training of the task policy we sample a batch of starting states from the replay buffer and generate a trace for each starting state, by rolling out the world model with the task policy. These traces are then used to compute TD-$\lambda$ targets for the policy gradient (Eq.~\ref{eq:policygradient}). Our goal is to identify counter-example traces and use them to directly modify the policy gradient in such a way that unsafe (but possibly rewardful) behaviours are not encouraged. For a given finite trace $\tau = s_t \to s_{t+1} \to \ldots \to s_{t+n}$ we denote the (discounted) cost from $s_t$ as follows,
\begin{equation}
    \textrm{cost}(s_t) = \sum^{n}_{i=t} \gamma^{i-t-1} \cdot C \cdot \mathbbm{1}[s_{i} \not \models \Psi] 
\end{equation}
where $C>0$ is a hyperparameter denoting the cost incurred when $\Psi$ is violated. If the cost of a trace is above a particular threshold then it is considered a counter-example. In particular, if $\textrm{cost}(s_t) \geq \gamma^{n-1} \cdot C$, this implies that the trace $\tau$ does not satisfy bounded safety (i.e.~$\exists \, i \; t < i \leq t+n \; s.t. \; s_{i} \not \models \Psi$), and so we consider it a counter-example. Based on the same intuition used in Sec.~\ref{sec:plpg}, we scale the gradients of these counter-examples to zero, so as not to encourage unsafe behaviours. Rather than a hard decision margin, we consider the following sigmoid function centred at $\gamma^{n-1} \cdot C$,
\begin{equation}
    w(x) = \frac{1}{1 + \exp(- \frac{1}{\kappa}(x - \gamma^{n-1}\cdot C))}
\end{equation}
where $\kappa>0$ is some scaling hyperparameter that determines how steep the sigmoid curve is. Let $W_t = w(\textrm{cost}(s_t))$, we then modify the policy gradient as follows,
\begin{equation}
\nabla \mathcal{J} = \mathbb{E}_{\pi_{\text{task}}} \Big[ \big( \sum^H_{t=1} sg(1 - W_t) \cdot G_t - \alpha \cdot G^C_t \big) \cdot \nabla \log \pi_{\text{task}}(a_t \mid s_t)\Big]
\label{eq:copt}
\end{equation}
where again $sg(\cdot)$ is the stop gradient function. Note here that we have also added the same penalty term (i.e.,~$- \alpha G_t^C$) from before. We demonstrate in our experiments that these further modifications to the policy gradient, in addition to the penalty term ($- \alpha G_t^C$), are important for stable convergence in the long run.


\section{Experimental Results}

In this section we present the results of running AMBS on three separate Safety Gym environments, and we compare each of the three penalty techniques described in Sec.~\ref{sec:penalty}. Hyperparameter settings and additional results can be found in Appendix \ref{sec:hyperparameters} and \ref{sec:additionalresults} respectively.

\paragraph{Safety Gym} We evaluate each algorithm on the following vision based environments from the Safety Gym benchmark \cite{ray2019benchmarking}: {\em PointGoal1}, {\em PointGoal2} and {\em CarGoal1}. At each timestep the agent receives a pixel observation $o_t$, which corresponds to a $3 \times 64 \times 64$ dimensional tensor. The agent then picks an action $a_t \in [-1,+1]^2$, which corresponds to rotational or linear actuator forces. In all environments the action space and underlying state space are continuous. Fig.~\ref{fig:envs} presents an overview of each of the environments, including the observations and corresponding safety-formula $\Psi$. In Appendix \ref{sec:safetygym} we provide additional details regarding the environments and vehicles used in our experiments.

\begin{figure}[!b]
    \centering
    \hfill
    \begin{subfigure}[t]{0.1\textwidth}
        \centering
        \includegraphics[width=\textwidth]{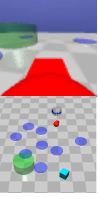}
        \caption{PointGoal1 \\ $\Psi = \neg hazard$}
        \label{fig:pointgoal1}
    \end{subfigure}
    \hfill
    \begin{subfigure}[t]{0.1\textwidth}
        \centering
        \includegraphics[width=\textwidth]{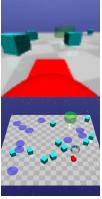}
        \caption{PointGoal2 \\ $\Psi = \neg hazard \land \neg collision$}
        \label{fig:pointgoal2}
    \end{subfigure}
    \hfill
    \begin{subfigure}[t]{0.1\textwidth}
        \centering
        \includegraphics[width=\textwidth]{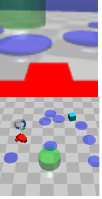}
        \caption{CarGoal1 \\ $\Psi = \neg hazard$}
        \label{fig:cargoal1}
    \end{subfigure}
    \caption{SafetyGym environments.}
    \label{fig:envs}
\end{figure}

\paragraph{Comparisons} We run AMBS \cite{goodallb2023approximate} with each of the three penalty techniques separately (i.e., PENL, PLPG and COPT), as a baseline we use a version of DreamerV3 \cite{hafner2023mastering} that implements the Augmented Lagrangian penalty framework (LAG) \cite{wright2006numerical, as2022constrained}, our baseline is similar to the algorithm proposed in \cite{huang2023safe} which is also built on DreamerV3. For completeness we also provide additional results for `vanilla' DreamerV3 \cite{hafner2023mastering} and `vannila' AMBS \cite{goodallb2023approximate} in Appendix \ref{sec:additionalresults}.

\begin{figure}[t]
    \centering
    \begin{subfigure}{0.235\textwidth}
         \centering
         \includegraphics[width=\textwidth]{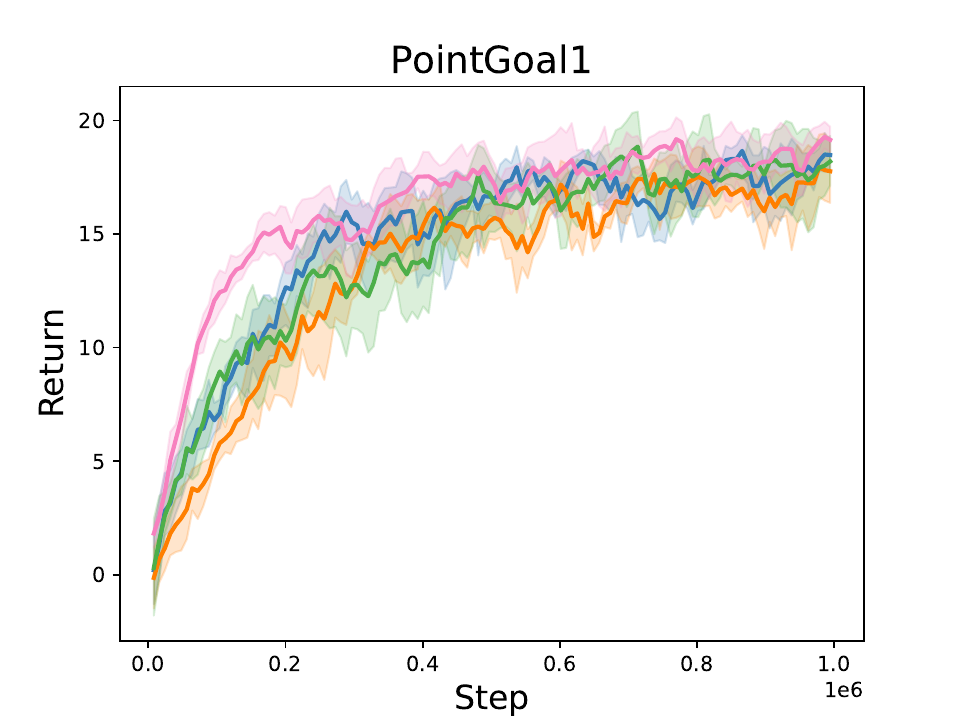}
     \end{subfigure}
     \hfill
     \begin{subfigure}{0.235\textwidth}
         \centering
         \includegraphics[width=\textwidth]{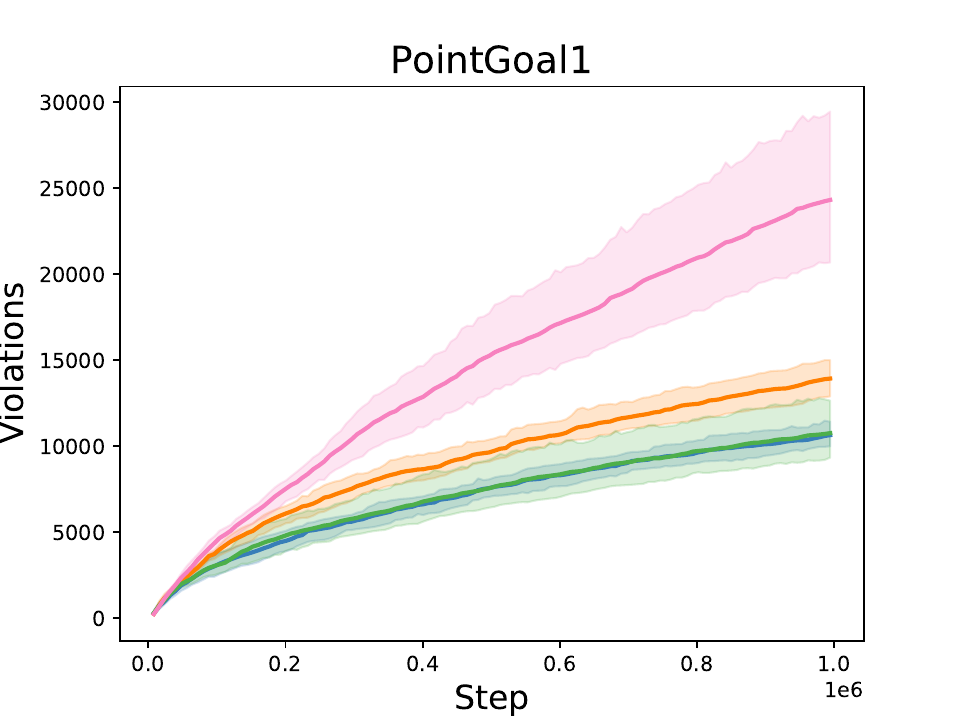}
     \end{subfigure}
     \hfill
     \begin{subfigure}{0.235\textwidth}
         \centering
         \includegraphics[width=\textwidth]{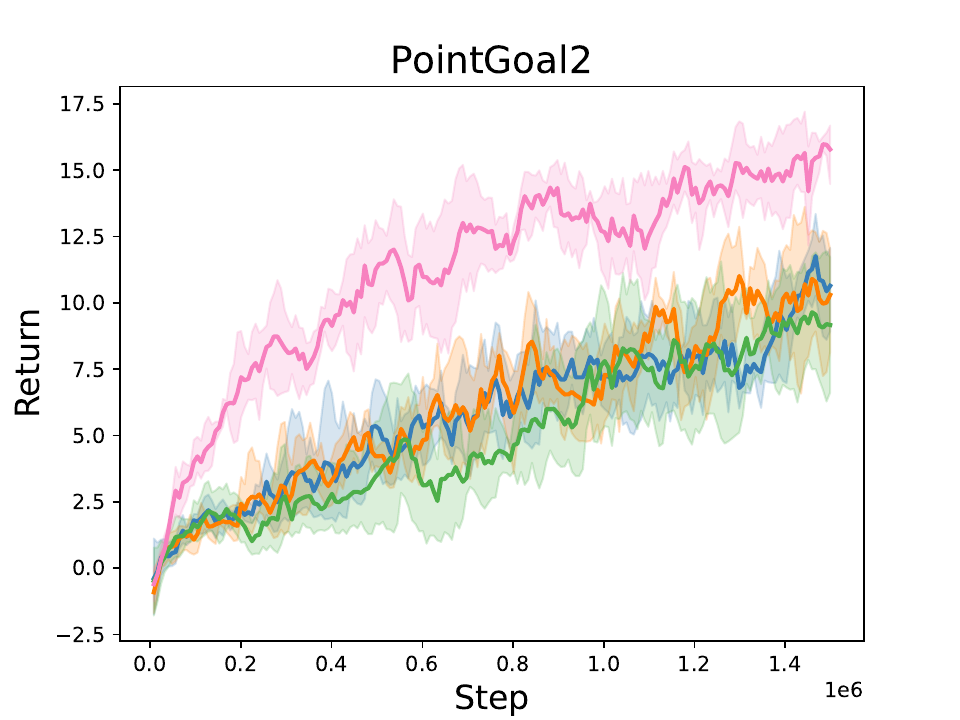}
     \end{subfigure}
     \hfill
     \begin{subfigure}{0.235\textwidth}
         \centering
         \includegraphics[width=\textwidth]{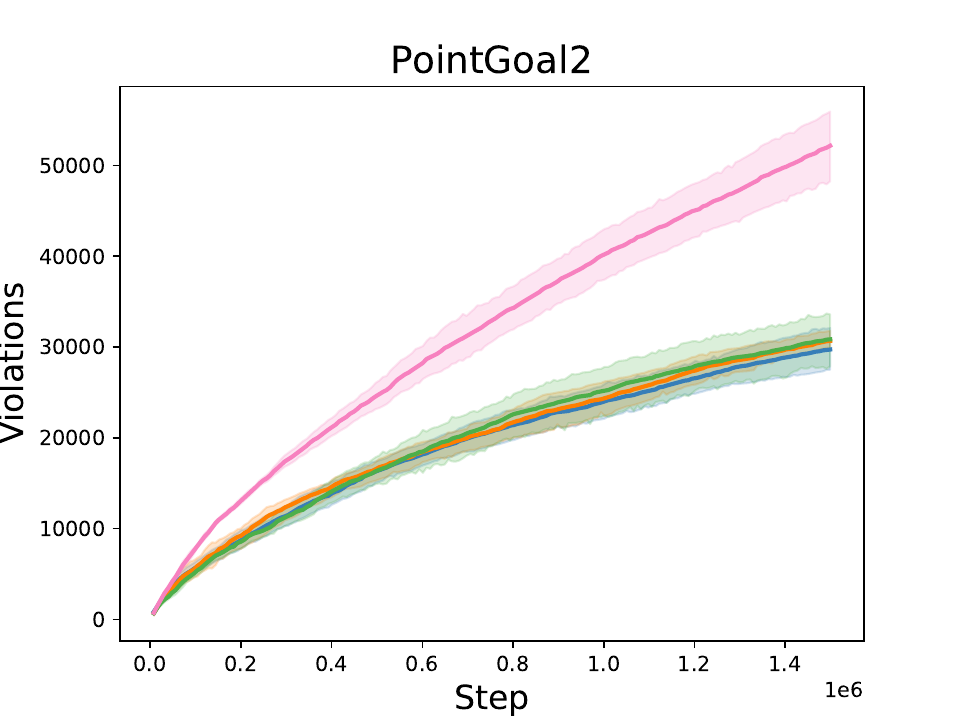}
     \end{subfigure}
    \hfill
     \begin{subfigure}{0.235\textwidth}
         \centering
         \includegraphics[width=\textwidth]{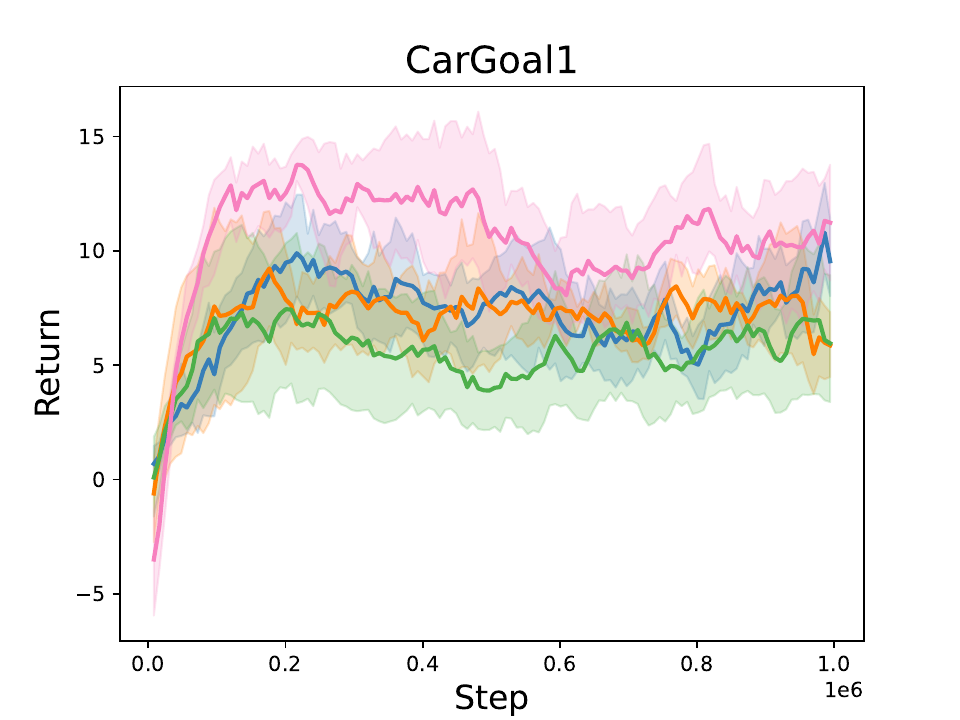}
     \end{subfigure}
     \hfill
     \begin{subfigure}{0.235\textwidth}
         \centering
         \includegraphics[width=\textwidth]{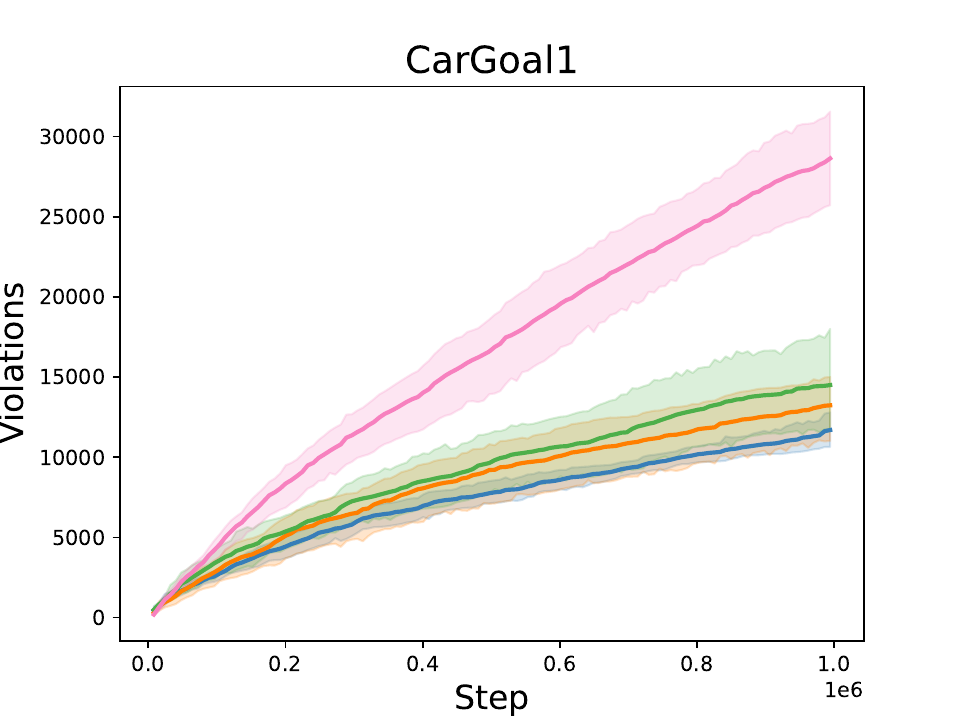}
     \end{subfigure}
     \hfill
    \begin{subfigure}{0.45\textwidth}
        \centering
        \includegraphics[width=\textwidth]{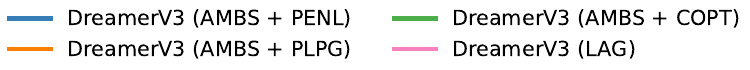}
    \end{subfigure}
    \caption{Episode return (left) and cumulative violations (right) for PointGoal1, PointGoal2 and CarGoal1.}
    \label{fig:safetygymresults}
\end{figure}

\paragraph{Training Details} Each agent is trained on a single Nvidia Tesla A30 (24GB RAM) GPU and a 24-core/48 thread Intel Xeon CPU with 256GB RAM. For the PointGoal1 and CarGoal1 environments we trained each agent for 1M frames. For the PointGoal2 environment we trained each agent for 1.5M frames, since this environment is harder to master. Each run is averaged over 5 seeds, and 95\% confidence intervals are reported both in the plots and table of results.

\paragraph{Results} We plot the accumulated episode return and cumulative number of safety-violations during training, which can be interpreted as the cost regret \cite{ray2019benchmarking}. Fig.~\ref{fig:safetygymresults} and Tab.~\ref{tab:results} present the results. We see that across all environments our methods outperform the baseline w.r.t.~the cumulative number of safety-violations. In terms of episode return, our methods clearly exhibit slower convergence than DreamerV3 with LAG. However, this is a trade-off we would expect to be present in these Safety Gym environments. In the PointGoal1 environment our methods converge to a policy with similar performance (w.r.t.~episode return), while committing far fewer safety violations than the baseline. In the PointGoal2 environment it looks as if none of the methods have converged and our methods continue to monotonically improve throughout the 1.5M frames. In the CarGoal1 environment all the methods exhibit some policy instability as the {\em Car} vehicle is harder to navigate with. Perhaps further investigation is need in this environment. Comparing between the three penalty techniques; clearly they all exhibit similar behaviour during training with the simple penalty critic obtaining slightly better results in terms of cumulative violations.

\paragraph{Comparing Convergence in the Long Run} In Safety Gym \cite{ray2019benchmarking} it is typical for model-free algorithms to take 10M frames to converge. However, since \cite{as2022constrained} demonstrated far superior sample complexity with model-based RL, it has become clear that model-based algorithms can converge within 1M-2M frames in most of the environments in Safety Gym. Given that we do exhibit slower convergence than the baseline (LAG). It might be interesting to analyze the behaviour of our algorithms for longer training runs. As is demonstrated in ~Fig.~\ref{fig:longresults}, AMBS with the simple penalty critic appears to diverge during longer training runs. However, the more principled penalty techniques, i.e.~PLPG, COPT and LAG maintain more stable convergence properties. We stress that more work is required to understand the convergence properties of PLPG and COPT and what guarantees they have if any.


\begin{figure}[t]
    \centering
    \begin{subfigure}{0.235\textwidth}
         \centering
         \includegraphics[width=\textwidth]{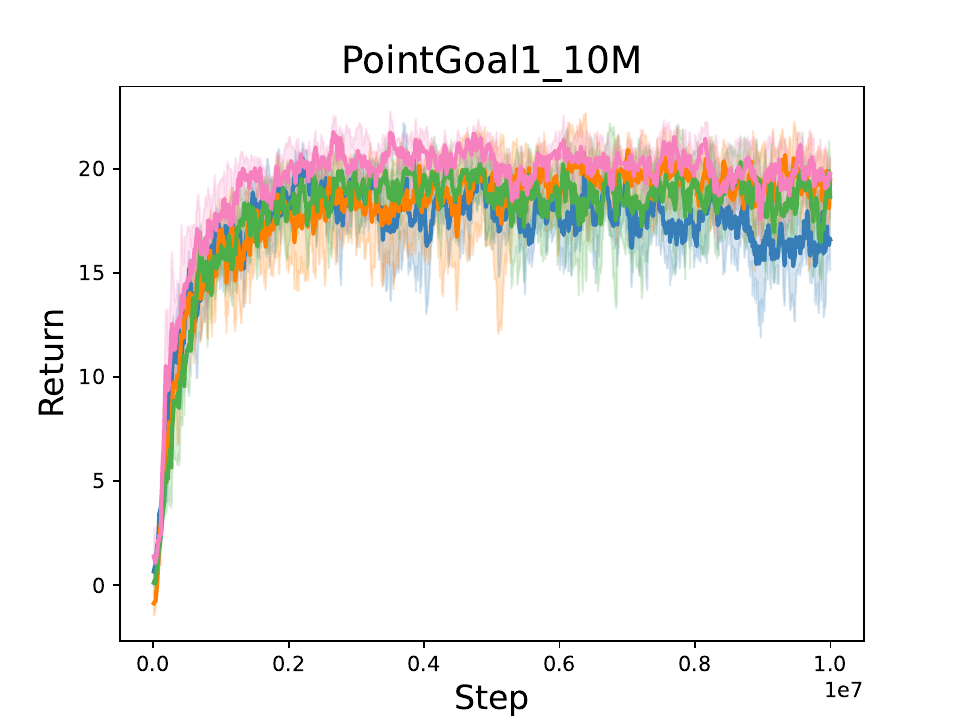}
     \end{subfigure}
     \hfill
     \begin{subfigure}{0.235\textwidth}
         \centering
         \includegraphics[width=\textwidth]{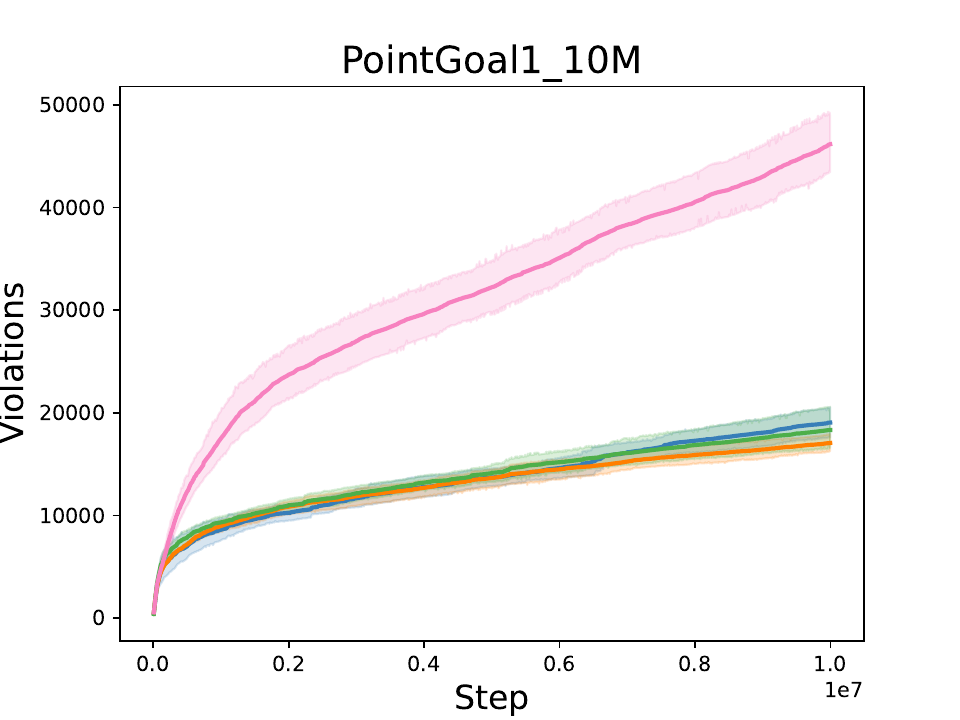}
     \end{subfigure}
     \hfill
     \begin{subfigure}{0.45\textwidth}
        \centering
        \includegraphics[width=\textwidth]{images/legend.pdf}
    \end{subfigure}
    \caption{Long run (10M frames) episode return (left) and cumulative violations (right) for PointGoal1.}
    \label{fig:longresults}
\end{figure}

\begin{table*}[!h]
\centering
\caption{Episode return and cumulative violations at the end of training.}
\begin{tabular}{| M{1.5cm} M{2.2cm} | P{2.4cm} P{2.4cm} P{2.4cm} P{2.4cm} |}
     \hline
      & & AMBS + PENL & AMBS + PLPG & AMBS + COPT & DreamerV3 + LAG  \\
     \hline
     \multirow{2}{6em}{PointGoal1 (1M)} & Episode Return $\uparrow$ & $17.32 \pm 3.29$ & $17.76 \pm 2.18$ & $18.19 \pm 1.72$ & $\mathbf{19.15 \pm 0.92}$ \\
  & \# Violations $\downarrow$ & $\mathbf{9354 \pm 3734}$ & $13937 \pm 1722$ & $10766 \pm 2877$ & $24996 \pm 6627$ \\
  \hline
     \multirow{2}{6em}{PointGoal2 (1.5M)} & Episode Return $\uparrow$ & $10.64 \pm 2.61 $ & $ 10.30 \pm 3.44$ & $9.16 \pm 4.07$ & $\mathbf{15.78 \pm 1.84}$ \\
  & \# Violations $\downarrow$ & $\mathbf{29720 \pm 3850}$  & $30673 \pm 1800$ & $ 30839 \pm 4647$ & $52157 \pm 6151$ \\
  \hline
  \multirow{2}{6em}{CarGoal1 \\ (1M)} & Episode Return $\uparrow$ & $8.87 \pm 2.95$ & $5.86 \pm 2.30$ & $5.96 \pm 4.15$ & $\mathbf{11.23 \pm 4.10}$ \\
  & \# Violations $\downarrow$ & $\mathbf{11423 \pm 1479}$ & $13236 \pm 3294$ & $14500 \pm 4675$ & $28639 \pm 4644$ \\
  \hline
  \multirow{2}{6em}{PointGoal1 (10M)} & Episode Return $\uparrow$ & $16.60 \pm 2.23$ & $19.45 \pm 1.62$ & $18.66 \pm 2.15$ & $\mathbf{19.74 \pm 1.43}$ \\
  & \# Violations $\downarrow$ & $19039 \pm 2339$ & $\mathbf{17049 \pm 1321}$ & $18320 \pm 3080$ & $46153 \pm 4637$ \\
     \hline
\end{tabular}
\label{tab:results}
\end{table*}



\section{Related Work}

\textit{Safe RL} is concerned with the problem of learning policies that maximise expected accumulated reward in environments where it is expected to comply with some minimum standard of safe behaviour, typically during training and more importantly during deployment of the agent \cite{garcia2015comprehensive}. Safe RL is an important problem if we are to deploy RL systems on meaningful tasks in the real world, where it is necessary to avoid harmful or catastrophic situations \cite{amodei2016concrete}. Thus, before deploying systems in the real world, it is necessary to determine whether they can balance reward and constraint satisfaction on benchmarks like Safety Gym \cite{ray2019benchmarking}. 

In this paper we formulate the safe RL problem by using propositional safety-formulas \cite{goodalla2023approximate, goodallb2023approximate} and safety properties such as bounded safety \cite{giacobbe2021shielding}. However, CMDPs \cite{altman1999constrained} are certainly a more common way of formulating the safe RL problem. In CMDPs, cumulative discounted or episodic cost constraints are typically considered. Moreover, chance constraints \cite{ono2015chance}, probabilistic constraints, and more \cite{chow2017risk, xiong2023provably, agnihotri2023average} have also been considered. 

For CMDPs, several model-free approaches to safe RL have been proposed \cite{ray2019benchmarking, chow2017risk, achiam2017constrained}, including those based on trust region methods \cite{schulman2015trust, duan2016benchmarking} and Lagrange relaxations. Constrained policy optimisation (CPO) \cite{achiam2017constrained} is based on trust region policy optimisation (TRPO) \cite{schulman2015trust}, which optimises a lower bound on policy performance, ensuring monotonic policy improvement updates that are close to the original policy. An additional repair step is used to ensure policy updates result in safe policies. Lagrange methods formulate a relaxation of the corresponding constrained problem to obtain a smooth optimisation problem. The Augmented Lagrangian \cite{wright2006numerical} is a popular method for safe RL based on adaptive penalty coefficients. Unfortunately, model-free methods are particularly sample inefficient, which for safety-critical systems is an issue. As a result, several model-based approaches have been proposed \cite{berkenkamp2017safe, thomas2021safe}, including those based on the Dreamer architectures \cite{as2022constrained, huang2023safe}, which have demonstrated far superior sample complexity on benchmarks like Safety Gym, including here.

\paragraph{Shielding} For a given temporal logic formula that encodes the constraints of the system, the goal of shielding is to prevent the agent from entering any unsafe states (i.e., those that violate the given formula), by constructing a reactive shield that overrides actions proposed by the agent when necessary. The original {\em Shielding for RL} \cite{alshiekh2018safe} came with strong guarantees, i.e., {\em correctness} and {\em minimal interference}, at the cost of quite restrictive assumptions. Since then, a significant amount of work on shielding has been concerned with either relaxing these assumptions to some extent, or applying a similar version of the original approach to simpler domains where these assumptions can be met. 

For example, {\em bounded prescience shielding} (BPS) \cite{giacobbe2021shielding} assumes access to a black-box simulator of the environment for bounded look-ahead shielding of the learned policy. {\em Latent shielding} \cite{he2021androids} and AMBS \cite{goodallb2023approximate} are closely related, both rely on world models \cite{hafner2019learning} to shield the learned policy, although AMBS is a more general framework that comes with probabilistic safety guarantees which we extend to the continuous state and action spaces in this paper. 
\balance

Other recent methods include PLPG \cite{lee2023learning}, which is based on probabilistic logic programs and comes with convergence guarantees. However, a suitable safety abstraction of the system with known probabilities is still assumed. Shielding has also been applied in partially observable settings \cite{carr2022safe}, resource constrained POMDPs \cite{ajdarow2022shielding}, multi-agent systems \cite{xiao2023model}, and for online learning tabular and parametric shields \cite{shperberg2022learning}.


\section{Conclusions}

In this paper we extended AMBS \cite{goodallb2023approximate} to the continuous setting and conducted experiments on the Safety Gym benchmark \cite{ray2019benchmarking}. We established probabilistic guarantees for the continuous setting and proposed three penalty techniques that are essential for the convergence of AMBS in Safety Gym, as they provide the underlying policy with the relevant safety information, see Fig.~\ref{fig:naiveambs} ( Sec.~\ref{sec:introduction}) and Appendix \ref{sec:additionalresults}.

For the three penalty techniques, we obtained a robust set of results on three Safety Gym environments, including a comparison to a version of DreamerV3 that implements the Augmented Lagrangian \cite{wright2006numerical}, `vanilla' DreamerV3 \cite{hafner2023mastering} and `vanilla' AMBS \cite{goodallb2023approximate}. In all cases our methods obtain far fewer safety violations during training, at the cost of slower convergence. We claim in the long run, that our methods may converge to a policy with similar performance (w.r.t.~episode return) compared to the baseline. Although, this claim needs to be more thoroughly investigated. During shorter training runs our three penalty techniques exhibit similar performance, both in terms of episode return and cumulative violations. However, we demonstrate that the more sophisticated techniques (i.e.~PLPG and COPT) are useful for stable convergence during longer runs.


Compared with CMDPs \cite{altman1999constrained}, shielding approaches are policy agnostic. The safety of the system depends entirely on the shield, whereas for CMDP-based approaches the safety of the system depends on the learned policy. We stress the importance of the probabilistic guarantees that have been established for AMBS, as they allow practitioners to choose hyperparameter settings based on their domain-specific safety requirements and have confidence in the safety of their systems. While these guarantees make non-trivial assumptions, for particular settings we know (with high confidence) when these assumptions are satisfied from previously established sample complexity bounds. For partially observable settings the use of world models \cite{hafner2019learning, hafner2023mastering} also has good theoretical backing. 

Important future work includes, a more thorough investigation into the convergence properties of PLPG and COPT, and the key components of AMBS. It may also be interesting to establish stronger theoretical claims for some common continuous settings, for which sample complexity bounds already exist, and present it in a way that is compatible with AMBS.





\begin{acks}
This work was supported by UK Research and Innovation [grant number EP/S023356/1], in the UKRI Centre for Doctoral Training in Safe and Trusted Artificial Intelligence (\url{www.safeandtrustedai.org}).
\end{acks}



\bibliographystyle{ACM-Reference-Format} 
\bibliography{sample}

\onecolumn

\clearpage
\newpage

\appendix

\section{Algorithms}
\label{sec:algorithms}
\begin{algorithm}[!htb]
\caption{DreamerV3 \cite{hafner2023mastering} with AMBS \cite{goodallb2023approximate}}
\raggedright
\label{alg:ambs}
\textbf{Initialise:} replay buffer $\mathcal D$ with $R$ random episodes and RSSM parameters $\theta$ randomly.\\
\begin{algorithmic}
\While{not converged}
\State \textit{// World model learning}
\State Sample a batch $B$ of transition sequences $ \{ \langle o_t, a_t, r_t, c_t, \gamma_t, o_{t+1} \rangle^{k+H}_{t=k} \} \sim \mathcal{D}$.
\State Compute the components $h_{t}, z_t, \hat z_{t}, \hat r_t, \hat c_t, \hat \gamma$ and update the RSSM parameters $\theta$ with maximum likelihood.
\State \textit{// Task policy optimisation}
\State Using the task policy $\pi_{\text{task}}$ `imagine' sequences $\{h_{t:t+H}, \hat z_{t:t+H}, \hat r_{t:t+H}, \hat c_{t:t+H}, \hat \gamma_{t:t+H}  \}$ from every starting observation $o_t \in B$.
\State Update the parameters of $\pi_{\text{task}}$ with the penalty policy gradient (Eq.~\ref{eq:penaltypolicygradient}) or PLPG (Eq.\ref{eq:plpg}) or COPT (Eq.~\ref{eq:copt}) .
\State \textit{// Safety critic optimisation}
\State Train the TD3-style \cite{fujimoto2018addressing} safety critics $v^C_1$ and $v^C_2$ with maximum likelihood to estimate $\mathbb{E}_{\pi_{\text{task}}}[ \sum^{\infty}_{t=1} \gamma^{t-1} \cdot c_t ]$.
\State \textit{// Safe policy optimisation}
\State Using the safe policy $\pi_{\text{safe}}$ `imagine' sequences $\{h_{t:t+H}, \hat z_{t:t+H}, \hat r_{t:t+H}, \hat c_{t:t+H}, \hat \gamma_{t:t+H}  \}$ from every starting observation $o_t \in B$.
\State Update the parameters of $\pi_{\text{safe}}$ with the policy gradient, i.e.~$\nabla \mathcal{J} = \mathbb{E}_{\pi_{\text{safe}}} \Big[ \sum^H_{t=1} G^C_t \cdot \nabla \log \pi_{\text{safe}}(a_t \mid s_t)\Big]$.
\State \textit{// Environment interaction}
\For{$t = 1, ..., L$} 
\State From $o_t$ compute $\hat s_t = (z_t, h_t)$ and sample an action $a \sim \pi_{\text{task}}$ with the task policy.
\State With Alg.~\ref{alg:shield} check if $s_t \models \mathbb{P}_{\geq 1 -  \Delta}(\square^{\leq n} \Psi )$ and play the shielded action $a'$ in the environment.
\State Observe $r_t, o_{t+1}$ and $L(s_t)$ from the environment and compute $c_t = C$ if $s_t \models \Psi$ else $c_t=0$.
\State Append $\langle o_t, a_t, r_t, c_t, \gamma_t, o_{t+1} \rangle$ to $\mathcal{D}$.
\EndFor
\EndWhile
\end{algorithmic}
\end{algorithm}

\begin{algorithm}[!htb]
\caption{Shielding Procedure (AMBS) \cite{goodallb2023approximate}}
\raggedright
\label{alg:shield}
\textbf{Input:} approximation error $\epsilon$, desired safety level $\Delta$, current state $\hat s = (z, h)$, proposed action $a$, `imagination' horizon $H$, look-ahead shielding horizon $T$, number of samples $m$, RSSM $p_{\theta}$, safe policy $\pi_{\text{safe}}$, task policy $\pi_{\text{task}}$ and safety critics $v^C_1$ and $v^C_2$ \\
\textbf{Output:} shielded action $a'$.
\begin{algorithmic}
\State \textit{// Sample Traces}
\For{$i = 1, ..., m$}
    \State From $\hat s_0 = (z, h)$ play $a_0 = a$ and using the task policy $\pi_{\text{task}}$ `imagine' sequences $\{h_{1:H}, \hat z_{1:H}, \hat r_{1:H}, \hat c_{1:H}, \hat \gamma_{1:H}  \}$.
    \State \textit{// Check whether each trace is satisfying}
    \State Compute $\textrm{cost}(\tau) = \sum^H_{t=1} (\hat \gamma_t)^{t-1} \hat c_t \;\; \text{\it or if safety critics} \;\; \textrm{cost}(\tau) = \sum^{H-1}_{t=1} (\hat\gamma_t)^{t-1} \hat c_t + \min\left\{v^C_1(h_H, \hat z_H), v^C_2(h_H, \hat z_H)  \right\}$.
    \State $X_i = \mathbbm{1}\left[\text{cost}(\tau) < \gamma^{T-1} \cdot C \right]$.
\EndFor
\State Let $\tilde \mu_{s \models \phi} = \frac{1}{m}\sum^{m}_{i=1} X_i$
\State \textbf{If} $\tilde \mu_{s \models \phi} \in [1 - \Delta + \epsilon, 1]$ \textbf{return} $a' = a$ \textbf{else} \textbf{return} $a' \sim \pi_{\text{safe}}$
\end{algorithmic}
\end{algorithm}

\section{Proofs}
\label{sec:proofs}

\subsection{Proof of Theorem \ref{prop:boundonm}}
The proof is almost identical to the one presented in \cite{goodallb2023approximate}, we provide it here for completeness.
\begin{restatedthm}{\ref{prop:boundonm} (restated)}

Let $\epsilon > 0$, $\delta > 0$, $s \in S$ be given. With access to the true transition system $\mathcal{T}$, with probability $1 - \delta$ we can obtain an $\epsilon$-approximate estimate of the measure $\mu_{s\models\phi}$, by sampling $m$ traces $\tau \sim \mathcal{T}$, provided that,

\begin{equation*}
    m \geq \frac{1}{2\epsilon^2} \log\left(\frac{2}{\delta}\right) 
\end{equation*}
\end{restatedthm}
\begin{proof}
We estimate $\mu_{s\models\phi}$ by sampling $m$ traces $\langle\tau_j\rangle^m_{j=1}$ from $\mathcal{T}$. Let $X_1, ..., X_m$ be indicator r.v.s such that,

\begin{equation*}
    X_j = \begin{cases}
        1 & \text{if $\tau_j \models \square^{\leq n} \Psi$,}\\
        0 & \text{otherwise}
    \end{cases}
\end{equation*}
Let,
\begin{equation*}
    \tilde \mu_{s\models\phi} = \frac{1}{m} \sum^m_{j=1}X_j, \; \text{where} \;\: \mathbb{E}_{\mathcal{T}}[\tilde \mu_{s\models\phi}] = \mu_{s\models\phi} 
\end{equation*}
Then by Hoeffding's inequality,
$$\mathbb{P}\left[|\tilde \mu_{s\models\phi} - \mu_{s\models\phi} | \geq \epsilon \right]\leq 2\exp\left( - 2 m \epsilon^2 \right)$$
Bounding the RHS from above with $\delta$ and rearranging gives the desired bound.\\

\noindent It remains to argue that $\tau_j \models \square^{\leq n} \Psi$ is easily checkable. Indeed this is the case (for polynomial $n$) because in the fully observable setting we have access to the state and so we can check that $\forall i \; \tau[i] \models \Psi$.
\end{proof}

\subsection{Proof of Theorem \ref{prop:kl}}

Rather than operate directly with KL divergence, for our purposes it is convenient to consider the variation distance (or total variation distance). The variation distance is a distance metric between two probability distributions that satisfies the triangle inequality. Whereas, the KL divergence can be interpreted as the information lost by approximating $\mathcal{T}(\cdot \mid s)$ with $\widehat{\mathcal{T}}(\cdot \mid s)$. The variation distance and KL divergence are linked by Pinsker's inequality \cite{tsybakov2004introduction}. \\

\noindent First consider a measurable space $(\Omega, \mathcal{F})$, where $\Omega$ may be an uncountable set (e.g. the reals), let $P$ and $Q$ be probability measures on the space $(\Omega, \mathcal{F})$, then the variation distance is defined as,
\begin{equation*}
    D_{\textrm{VAR}}(P, Q) = \sup_{A \in \mathcal{F}} \vert P(A) - Q(A)\vert
\end{equation*}
This is the most general definition. However for specific $\Omega$ we might consider equivalent definitions. For example, when $\Omega = \mathbb{R}$, we could consider the following definition,
\begin{equation*}
    D_{\textrm{VAR}}(P, Q) = \frac{1}{2} \int^{\infty}_{-\infty} |p(x) - q(x)|dx
\end{equation*}
Similarly, when $\Omega$ is finite or countably finite, we may consider the following definition related to the $L^{1}$ norm,
\begin{equation*}
    D_{\textrm{VAR}}(P, Q) = \Vert P - Q \Vert_1 = \frac{1}{2}\sum_{\omega \in \Omega} |P(\{\omega\}) - Q(\{ \omega\})|
\end{equation*}
We split the proof of Theorem \ref{prop:kl} into two parts. First we present the {\em error amplification} lemma \cite{rajeswaran2020game} adapted to our purposes, before providing the proof of Theorem \ref{prop:kl}.

\begin{lemma}[error amplification] Let $\mathcal{T}$ and $\widehat{\mathcal{T}}$ be two transition systems over the same same set of states $S$, and with the same initial state distribution $\iota_{init} (\cdot)$. Let $\mathcal{T}^t$ and $\widehat{\mathcal{T}}^t$ be the marginal state distribution at time $t$ for the transitions systems $\mathcal{T}$ and $\widehat{\mathcal{T}}$ respectively. Specifically let,
\begin{eqnarray*}
    \mathcal{T}^t(s) = \mathbb{P}_{\tau \sim \mathcal{T}}[\tau[t] = s] \\
    \widehat{\mathcal{T}}^t(s) = \mathbb{P}_{\tau \sim \widehat{\mathcal{T}}}[\tau[t] = s]
\end{eqnarray*}
Suppose we have,
\begin{equation*}
	D_{KL}\innerp*{\mathcal{T}(\cdot \mid s) ; \widehat{\mathcal{T}}(\cdot \mid s) } \leq \alpha \; \forall s \in S 
\end{equation*}
then the variation distance between marginal distributions $\mathcal{T}^t$ and $\widehat{\mathcal{T}}^t$ is bounded as follows,
\begin{equation*}
    D_{\textrm{VAR}}\left(\mathcal{T}^t,  \widehat{\mathcal{T}}^t\right) \leq t \cdot \sqrt{\frac{1}{2}\alpha}
\end{equation*}
\label{lemma:marginal}
\end{lemma}
\begin{proof}
First we fix some $s \in S$, we can bound the absolute difference between the probabilities $\mathcal{T}^t(s)$ and $\widehat{\mathcal{T}}^t(s)$ as follows,
\begin{align*}
        \left\vert \mathcal{T}^t(s) - \widehat{\mathcal{T}}^t(s) \right\vert & = \bigg\vert \int_{\bar s \in S} \mathcal{T}(s \mid \bar s) \mathcal{T}^{t-1}(\bar s) d \bar s - \int_{\bar s \in S} \widehat{\mathcal{T}}(s \mid \bar s) \widehat{\mathcal{T}}^{t-1}(\bar s) d \bar s   \bigg\vert \\
        &\leq \int_{\bar s \in S} \left\vert \mathcal{T}(s \mid \bar s) \mathcal{T}^{t-1}(\bar s) - \widehat{\mathcal{T}}(s \mid \bar s) \widehat{\mathcal{T}}^{t-1}(\bar s)\right\vert d \bar s\\
        & \leq \int_{\bar s \in S} \bigg( \left\vert \mathcal{T}(s \mid \bar s) \mathcal{T}^{t-1}(\bar s) - \mathcal{T}(s \mid \bar s) \widehat{\mathcal{T}}^{t-1}(\bar s) \right\vert \\
        & \qquad + \left\vert \mathcal{T}(s \mid \bar s) \widehat{\mathcal{T}}^{t-1}(\bar s) - \widehat{\mathcal{T}}(s \mid \bar s) \widehat{\mathcal{T}}^{t-1}(\bar s) \right\vert
        \bigg) d \bar s \\
        & = \int_{\bar s \in S} \bigg( \mathcal{T}(s \mid \bar s) \left\vert \mathcal{T}^{t-1}(\bar s) - \widehat{\mathcal{T}}^{t-1}(\bar s) \right\vert \\
        & \qquad + \widehat{\mathcal{T}}^{t-1}(\bar s) \left\vert \mathcal{T}(s \mid \bar s) - \widehat{\mathcal{T}}(s \mid \bar s)\right\vert  \bigg) d \bar s\\
\end{align*}
where the first inequality is straightforward, the second inequality comes from simultaneously adding and subtracting the term $\mathcal{T}(s \mid \bar s) \widehat{\mathcal{T}}^{t-1}(\bar s)$ and applying the triangle inequality. Now we can bound the variation distance for the full marginal distributions  $\mathcal{T}^t$ and $\widehat{\mathcal{T}}^t$,
\begin{align*}
    D_{\textrm{VAR}}\left(\mathcal{T}^t,  \widehat{\mathcal{T}}^t\right) &= \frac{1}{2}\int_{s \in S} \left\vert \mathcal{T}^t(s) - \widehat{\mathcal{T}}^t(s) \right\vert ds \\
    &\leq \frac{1}{2}\int_{s \in S} \Bigg( \int_{\bar s \in S} \bigg( \mathcal{T}(s \mid \bar s) \left\vert \mathcal{T}^{t-1}(\bar s) - \widehat{\mathcal{T}}^{t-1}(\bar s) \right\vert \\
    & \qquad + \widehat{\mathcal{T}}^{t-1}(\bar s) \left\vert \mathcal{T}(s \mid \bar s) - \widehat{\mathcal{T}}(s \mid \bar s)\right\vert \bigg) d \bar s\Bigg) ds \qquad \text{(using the previous bound)}\\
    & \leq \frac{1}{2} \int_{\bar s \in S} \left( \left\vert \mathcal{T}^{t-1}(\bar s) - \widehat{\mathcal{T}}^{t-1}(\bar s) \right\vert \int_{s \in S} \mathcal{T}(s \mid \bar s) ds \right) d \bar s \\
    & \qquad + \frac{1}{2} \int_{\bar s \in S} \left(  \widehat{\mathcal{T}}^{t-1}(\bar s) \int_{s \in S} \left\vert \mathcal{T}(s \mid \bar s) - \widehat{\mathcal{T}}(s \mid \bar s)\right\vert ds \right) d \bar s \\
    & \leq D_{\textrm{VAR}}\left(\mathcal{T}^{t-1},  \widehat{\mathcal{T}}^{t-1}\right) + \max_{\bar s \in S} \left\{ D_{\textrm{VAR}}\left( \mathcal{T}(\cdot \mid \bar s), \widehat{\mathcal{T}}(\cdot \mid \bar s) \right) \right\} \\
    & \leq t \cdot \max_{\bar s \in S} \left\{ D_{\textrm{VAR}}\left( \mathcal{T}(\cdot \mid \bar s), \widehat{\mathcal{T}}(\cdot \mid \bar s) \right) \right\} \qquad \text{(induction on $t$)}\\
    & \leq t \cdot \sqrt{\frac{1}{2} D_{KL}\innerp*{\mathcal{T}(\cdot \mid s) ; \widehat{\mathcal{T}}(\cdot \mid s) }} \qquad \text{(using Pinsker's ineq.)}\\
    & \leq t \cdot \sqrt{\frac{1}{2}\alpha} \qquad \text{(main assum.)}\\
\end{align*}
where the second inequality comes from rearranging the integrals and the third inequality comes from the fact that some of the integrals sum to 1. The induction on $t$ is also not immediately straightforward, although if we consider the base case i.e.~$D_{\textrm{VAR}}\left(\mathcal{T}^{1},  \widehat{\mathcal{T}}^{1}\right) \leq   \max_{\bar s \in S} \left\{ D_{\textrm{VAR}}\left( \mathcal{T}(\cdot \mid \bar s), \widehat{\mathcal{T}}(\cdot \mid \bar s) \right) \right\}$, then it should not be too hard to see.
\end{proof}
\begin{restatedthm}{\ref{prop:kl} (restated)}
Suppose that for all $s \in S$, the Kullback-Leibler (KL) divergence
and $\widehat{\mathcal{T}}(\cdot \mid s)$ is upper-bounded by some $\alpha \leq \epsilon^2/(2 n^2)$. That is,
\begin{equation*}
	D_{KL}\innerp*{\mathcal{T}(\cdot \mid s) ; \widehat{\mathcal{T}}(\cdot \mid s) } \leq \alpha \; \forall s \in S 
\end{equation*}
Now fix an $s \in S$ and let $\epsilon > 0$, $\delta> 0$ be given. With probability $1 - \delta$ we can obtain an $\epsilon$-approximate estimate of the measure $\mu_{s\models\phi}$, by sampling $m$ traces $\tau \sim \widehat{\mathcal{T}}$, provided that,
\begin{equation*}
    m \geq \frac{2}{\epsilon^2} \log\left(\frac{2}{\delta}\right)
\end{equation*}
\end{restatedthm}
\begin{proof}
First recall the following definition,
\begin{equation*}
    \mu_{s \models \phi} = \mu_s(\{\tau \mid \tau[0] = s, \text{ for all }0 \leq i \leq n, \tau[i] \models \Psi\})
\end{equation*}
where $\tau \sim \mathcal{T}$. Equivalently we can write,
\begin{equation*}
    \mu_{s \models \phi} = \mathbb{P}_{\tau \sim \mathcal{T}}[\tau \models \square^{\leq n} \Psi \mid \tau[0] = s ]
\end{equation*}
Similarly, let $\hat \mu_{s \models \phi}$ be defined as the {\em true} probability under $\widehat{\mathcal{T}}$,
\begin{equation*}
    \hat \mu_{s \models \phi} = \mathbb{P}_{\tau \sim \widehat{\mathcal{T}}}[\tau \models \square^{\leq n} \Psi \mid \tau[0] = s]
\end{equation*}
We first want to show that the absolute difference between $\mu_{s \models \phi}$ and $\hat \mu_{s \models \phi}$ is upper bounded by $\epsilon/2$. Then we want to estimate $\hat \mu_{s \models \phi}$ up to $\epsilon/2$ error, which immediately gives us an $\epsilon$-approximation for $\mu_{s \models \phi}$. Let the following denote the average state distribution for $\mathcal{T}$ and $\widehat{\mathcal{T}}$ respectively,
\begin{eqnarray*}
    \rho_{\mathcal{T}}(\bar s) = \frac{1}{n}\sum^n_{i=1}\mathbb{P}_{\tau \sim \mathcal{T}}[\tau[i] = \bar s \mid \tau[0] = s] \\
    \rho_{\widehat{\mathcal{T}}}(\bar s) = \frac{1}{n}\sum^n_{i=1}\mathbb{P}_{\tau \sim \widehat{\mathcal{T}}}[\tau[i] = \bar s \mid \tau[0] = s] \\
\end{eqnarray*}
We obtain an upper bound for the absolute difference between $\mu_{s \models \phi}$ and $\hat \mu_{s \models \phi}$ by adapting the simulation lemma \cite{kearns2002near} as follows,
\begin{align*}
    \big\vert \mu_{s \models \phi} - \hat \mu_{s \models \phi} \big\vert & = \big\vert \mathbb{P}_{\tau \sim \mathcal{T}}[\tau \models \square^{\leq n} \Psi \mid \tau[0] = s ] - \mathbb{P}_{\tau \sim \widehat{\mathcal{T}}}[\tau \models \square^{\leq n} \Psi \mid \tau[0] = s ] \big\vert \\
    & \leq  1 \cdot D_{TV}(\rho_{\mathcal{T}}, \rho_{\widehat{\mathcal{T}}}) \qquad \text{(using the simulation lemma)} \\
    & = \frac{1}{2} \int_{\bar s \in S} \left\vert \rho_{\mathcal{T}}(\bar s) - \rho_{\widehat{\mathcal{T}}}(\bar s) \right\vert d \bar s\\
    & = \frac{1}{2n} \int_{\bar s \in S} \left\vert \sum^n_{i=1}  \mathbb{P}_{\tau \sim \mathcal{T}}[\tau[i] = \bar s \mid \tau[0] = s] - \mathbb{P}_{\tau \sim \widehat{\mathcal{T}}}[\tau[i] = \bar s \mid \tau[0] = s] \right\vert d \bar s \\
    & \leq \frac{1}{2n} \sum^n_{i=1} \int_{\bar s \in S}  \left\vert \mathbb{P}_{\tau \sim \mathcal{T}}[\tau[i] = \bar s \mid \tau[0] = s] - \mathbb{P}_{\tau \sim \widehat{\mathcal{T}}}[\tau[i] = \bar s \mid \tau[0] = s] \right\vert d \bar s \\
    & \leq \frac{1}{n} \sum^n_{i=1} \frac{1}{2}\int_{\bar s \in S}  \left\vert \mathcal{T}^i(\bar s) - \widehat{\mathcal{T}}^i(\bar s) \right\vert d\bar s \qquad \text{(by defn. see Lemma \ref{lemma:marginal})} \\
    & \leq \frac{1}{n} \sum^n_{i=1} n \cdot \sqrt{\frac{1}{2}\alpha} \qquad \text{(using Lemma \ref{lemma:marginal} and $ 1 \leq i \leq n$)} \\
    & = n \cdot \sqrt{\frac{1}{2}\alpha}
\end{align*}
Provided that $\alpha \leq \epsilon^2/(2 n^2)$ we now have $\big\vert \mu_{s \models \phi} - \hat \mu_{s \models \phi} \big\vert  \leq \epsilon/2$. It remains to obtain an $\epsilon/2$-approximation of $\hat \mu_{s \models \phi}$. With the exact same reasoning as in Theorem \ref{prop:boundonm}, we estimate $\hat \mu_{s \models \phi}$ by sampling $m$ traces $\langle\tau_j\rangle^m_{j=1}$ from $\widehat{\mathcal{T}}$, then provided,
\begin{equation*}
    m \geq \frac{2}{\epsilon^2} \log\left(\frac{2}{\delta}\right)
\end{equation*}
we obtain an $\epsilon/2$-approximation of $\hat \mu_{s \models \phi}$ with probability $1 - \delta$, which completes the proof.
\end{proof}

\subsection{Proof of Theorem \ref{prop:pomdp}}
\begin{restatedthm}{\ref{prop:pomdp} Restated}  Let $b_t$ be a latent representation (belief state) such that $p(s_t \mid o_{t\leq t}, a_{\leq t}) = p(s_t \mid b_t)$. Let the fixed policy $\pi(\cdot \mid b_t)$ be a general probability distribution conditional on belief states $b_t$. Let $f$ be a generic $f$-divergence measure (e.g., KL divergence). Then the following holds:
\begin{equation}
    D_{f}(\mathcal{T}(s' \mid b ), \widehat{\mathcal{T}}(s' \mid b )) \leq D_{f}(\mathcal{T}(b' \mid b ), \widehat{\mathcal{T}}(b' \mid b ))
\end{equation}
where $\mathcal{T}$ and $\widehat{\mathcal{T}}$ are the true and approximate transition system respectively, defined now over both states $s$ and belief states $b$.
\end{restatedthm}

\begin{proof}
The proof follows from \cite{goodallb2023approximate}. For clarity let the following probabilities be defined as follows,
\begin{align*}
    \mathcal{T}(s' \mid b) = \mathbb{P}_{\pi, p}[s_t = s' \mid b_{t-1} = b] \qquad & \qquad \mathcal{T}(b' \mid b) = \mathbb{P}_{\pi, p}[b_t = b' \mid b_{t-1} = b] \\ 
    \widehat{\mathcal{T}}(s' \mid b) = \mathbb{P}_{\pi, \hat p}[s_t = s' \mid b_{t-1} = b] \qquad & \qquad \widehat{\mathcal{T}}(b' \mid b) = \mathbb{P}_{\pi, \hat p}[b_t = b' \mid b_{t-1} = b]
\end{align*}
We can immediately define conditional and joint probabilities (e.g.~$\mathcal{T}(s', b' \mid b)$, $\mathcal{T}(s' \mid b)$, $\mathcal{T}(b' \mid s', b)$ and similarly for $\widehat{\mathcal{T}}$) using the standard laws of probability. We now apply the data-processing inequality \cite{ali1966general} for $f$-divergences to prove the upper bound,
\begin{align*}
    &D_{f}(\mathcal{T}(b' \mid b), \widehat{\mathcal{T}}(b' \mid b )) \\  
    &= \mathbb{E}_{b' \sim \widehat{\mathcal{T}}}\left[ f \left( \frac{\mathcal{T}(b' \mid b)}{\widehat{\mathcal{T}}(b' \mid b)}\right) \right] \\
    & = \mathbb{E}_{s', b' \sim \widehat{\mathcal{T}}}\left[ f \left( \frac{\mathcal{T}(s', b' \mid b)}{\widehat{\mathcal{T}}(s', b' \mid b)}\right) \right] \\
    & = \mathbb{E}_{s' \sim \widehat{\mathcal{T}}}\left[  \mathbb{E}_{b' \sim \widehat{\mathcal{T}}} f \left( \frac{\mathcal{T}(s', b' \mid b)}{\widehat{\mathcal{T}}(s', b' \mid b)}\right) \right] \\
    & \geq \mathbb{E}_{s' \sim \widehat{\mathcal{T}}}\left[ f \left( \mathbb{E}_{b' \sim \widehat{\mathcal{T}}}  \frac{\mathcal{T}(s', b' \mid b)}{\widehat{\mathcal{T}}(s', b' \mid b)}\right) \right] \qquad \text{(Jensen's)} \\
    & = \mathbb{E}_{s' \sim \widehat{\mathcal{T}}}\left[ f \left( \mathbb{E}_{b' \sim \widehat{\mathcal{T}}}  \frac{\mathcal{T}(s', b' \mid b) \widehat{\mathcal{T}}(b' \mid s', b)}{\widehat{\mathcal{T}}(s', b' \mid b)\mathcal{T}(b' \mid s', b)}\right) \right] \\
    & = \mathbb{E}_{s' \sim \widehat{\mathcal{T}}}\left[ f \left( \mathbb{E}_{b' \sim \widehat{\mathcal{T}}}  \frac{\mathcal{T}(s' \mid b)}{\widehat{\mathcal{T}}(s' \mid b)}\right) \right] \\
    & = \mathbb{E}_{s' \sim \widehat{\mathcal{T}}}\left[ f \left( \frac{\mathcal{T}(s' \mid b)}{\widehat{\mathcal{T}}(s' \mid b)}\right) \right] \\
    & = D_{f}(\mathcal{T}(s' \mid b ), \widehat{\mathcal{T}}(s' \mid b ))
\end{align*}
\end{proof}

\section{Derivations}
\label{sec:derivations}
\subsection{Derivation of Eq.~\ref{eq:plpg}}
We recall the policy gradient,
\begin{equation*}
    \nabla \mathcal{J} = \mathbb{E}_{\pi_{\text{task}}} \Big[ \sum^H_{t=1} G_t \cdot \nabla \log \pi_{\text{task}}(a_t \mid s_t)\Big]
\end{equation*}
Recall the re-normalising coefficient of the returns $G_t$,
\begin{equation*}
    \frac{\mathbb{P}_{\pi_{\text{safe}}}\left[s_t \models \square^{\leq H} \Psi \mid s_t, a_t \right]}{\mathbb{P}_{\pi_{\text{safe}}}\left[s_t \models \square^{\leq H} \Psi \mid s_t\right]}
\end{equation*}
We follow the control as inference framework \cite{levine2018reinforcement}, which introduces optimality random variables $\mathcal{O}_t \in \{0, 1\}$, where $\mathcal{O}_t = 1$ denotes that the action $a_t$ picked from state $s_t$ was optimal; the probability $p(\mathcal{O}_t = 1) = \exp(R(s_t, a_t))$. This is a valid probability for reward functions where $\forall (s, a) \; R(s, a) \leq 0$. \\

\noindent Now consider the following cost function: $c_t = 0$ if $s_t\models \Psi$ and $c_t = -C$ otherwise (for some $C > 0$). We introduce the safety random variables $\mathcal{S}_t \in \{0, 1\}$, where $\mathcal{S}_t = 1$ denotes that the action $a_t$ picked from state $s_t$ is safe. And $p(\mathcal{S}_t = 1) = \exp(c_t)$. Letting $V^{\pi_{\text{safe}}}_C(s) = \mathbb{E}_{\pi_{\text{safe}}}[\sum^{n}_{t=1} \gamma^{t-1} \cdot c_t \mid s_0 = s]$, we make the following observation,
\begin{align*}
    \frac{\mathbb{P}_{\pi_{\text{safe}}}\left[s_t \models \square^{\leq H} \Psi \mid s_t, a_t \right]}{\mathbb{P}_{\pi_{\text{safe}}}\left[s_t \models \square^{\leq H} \Psi \mid s_t\right]} & = \frac{\exp(c(s_t, a_t) + V^{\pi_{\text{safe}}}_C(s_{t+1}))}{\exp(V^{\pi_{\text{safe}}}_C(s_t))} \\
    & = \exp(c(s_t, a_t) + V^{\pi_{\text{safe}}}_C(s_{t+1}) - V^{\pi_{\text{safe}}}_C(s_t) ) \\
    & = \exp(\delta^{\text{safe}}_{t})
\end{align*}
where $\delta^{\text{safe}}_{t} = c(s_t, a_t) + V^{\pi_{\text{safe}}}_C(s_{t+1}) - V^{\pi_{\text{safe}}}_C(s_t)$. Now consider the penalty term $- \alpha \log \mathbb{P}_{\pi_{\text{task}}}[s_t \models \square^{\leq H} \Psi | s_t]$. Under the same control as inference framework we make the following observation,
\begin{align*}
    \nabla \log \mathbb{P}_{\pi_{\text{task}}}[s_t \models \square^{\leq H} \Psi | s_t] &= \nabla \log(\exp (V^{\pi_{\text{task}}}_C(s_t))) \\
    &= \nabla (V^{\pi_{\text{task}}}_C(s_t)) \\
    & \propto \mathbb{E}_{\pi_{\text{task}}} \left[ G_t^C \cdot \nabla \log \pi_{\text{task}}(a_t \: | \: s_t) \right] \qquad \text{(Policy Gradient thm.)}
\end{align*}
where $G_t^C$ is the cost return of the task policy and $V^{\pi_{\text{task}}}_C(s) = \mathbb{E}_{\pi_{\text{task}}}[\sum^{n}_{t=1} \gamma^{t-1} \cdot c_t \mid s_0 = s]$. Putting this all together gives us the following policy gradient, 
\begin{equation*}
	\nabla \mathcal{J} = \mathbb{E}_{\pi_{\text{task}}}\Big[  \sum^{\infty}_{t=0} \big( \exp(\delta^{\text{safe}}_{t}) \cdot G_t - \alpha \cdot G_t^C \cdot \big) \nabla\log \pi_{\text{task}}(a_t \: | \: s_t)  \Big]
\end{equation*}

\section{Hyperparameters}
\label{sec:hyperparameters}
\begin{table}[!htb]
    \begin{center}
    {\caption{DreamerV3 hyperparameters \cite{hafner2023mastering}. Other methods built on DreamerV3 such as AMBS \cite{goodallb2023approximate} and LAG use these hyperparameters unless otherwise specified.}
    \label{tab:dreamerv3}}   
    \begin{tabular}{| m{5cm} | P{3cm} | P{3cm} |}
    \hline
       Name  & Symbol & Value \\
       \hline
       \multicolumn{3}{|c|}{General}\\
       \hline
         Replay capacity& - & $10^6$\\
         Batch size & $B$ & 64\\
         Batch length & - & 16\\
         Number of envs & - & 8 \\
         Train ratio & - & 512 \\
         Number of encoder/decoder MLP layers & - & 2 \\
         Number of encoder/decoder MLP units & - & 1024 \\
         Activation & - & LayerNorm + SiLU \\
         \hline
         \multicolumn{3}{|c|}{World Model}\\
         \hline
         Number of latents & - & 32\\
         Classes per latent & - & 32\\
         Number of layers & - & 4 \\
         Number of hidden units & - & 768 \\
         Number of recurrent units & - & 2048 \\
         CNN depth & - & 64\\
         RSSM loss scales & $\beta_{\text{pred}}$, $\beta_{\text{dyn}}$, $\beta_{\text{rep}}$ & 1.0, 0.5, 0.1\\
         Predictor loss scales & $\beta_{o}, \beta_r, \beta_c, \beta_{\gamma}$  & 1.0, 1.0, 1.0, 1.0\\
         Learning rate & - & $10^{-4}$ \\
         Adam epsilon & $\epsilon_{\text{adam}}$& $10^{-8}$\\
         Gradient clipping & - & 1000 \\
         \hline
         \multicolumn{3}{|c|}{Actor Critic}\\
         \hline
         Imagination horizon & H & 15\\
         Discount factor & $\gamma$ & 0.997 \\
         TD lambda & $\lambda$ & 0.95\\
         Critic EMA decay & - & 0.98 \\
         Critic EMA regulariser & -&1\\
         Return norm. scale & $S_{\text{reward}}$ & $\text{Per}(R, 95) - \text{Per}(R, 5) $ \\
         Return norm. limit & $L_{\text{reward}}$ & 1\\
         Return norm. decay & - & 0.99\\
         Actor entropy scale & $\eta_{\text{actor}}$ & $3 \cdot 10^{-4}$\\
         Learning rate & - & $3 \cdot 10^{-5}$\\
         Adam epsilon & $\epsilon_{\text{adam}}$& $10^{-5}$\\
         Gradient clipping & - & 100 \\
         \hline
    \end{tabular}
    \end{center}
\end{table}

\begin{table}[!htb]
    \begin{center}
    {\caption{Hyperparameters for AMBS \cite{goodallb2023approximate} with penalty critic, PLPG or COPT policy gradients.}
    \label{tab:ambs}}
    \begin{tabular}{| m{5cm} | P{3cm} | P{3cm} |}
    \hline
       Name  & Symbol & Value \\
       \hline
       \multicolumn{3}{|c|}{Shielding}\\
       \hline
         Safety level& $\Delta$ & 0.1\\
         Approximation error & $\epsilon$ & 0.09\\
         Number of samples & $m$ & 512\\
         Failure probability & $\delta$ & 0.01 \\
         Look-ahead/shielding horizon & $T$ & 30\\
         Cost Value & $C$ & 10\\
         \hline
         \multicolumn{3}{|c|}{Policy Gradient}\\
         \hline
         Penalty coefficient & $\alpha$ & 1.0\\
         PLPG penalty coefficient & $\alpha_{\text{(PLPG)}}$ & 0.8 \\
         COPT penalty coefficient & $\alpha_{\text{(COPT)}}$ & 1.0 \\
         COPT sigmoid scale & $\kappa$ & 10.0\\
         \hline
         \multicolumn{3}{|c|}{Safe Policy}\\
         \hline
         \multicolumn{3}{|c|}{See `Actor Critic' in Table \ref{tab:dreamerv3}} \\
         \multicolumn{3}{|c|}{...} \\
         \hline
         \multicolumn{3}{|c|}{Safety Critic}\\
         \hline
         Type & - & TD3-style \cite{fujimoto2018addressing}\\
         Slow update frequency & - & 1\\
         Slow update fraction & - & 0.02\\
         EMA decay & - & 0.98 \\
         EMA regulariser & -&1\\
         Cost norm. scale & $S_{\text{cost}}$ & $\text{Per}(R, 95) - \text{Per}(R, 5) $ \\
         Cost norm. limit & $L_{\text{cost}}$ & 1\\
         Cost norm. decay & - & 0.99\\
         Learning rate & - & $3 \cdot 10^{-5}$\\
         Adam epsilon & $\epsilon_{\text{adam}}$& $10^{-5}$\\
         Gradient clipping & - & 100 \\
         \hline
         \multicolumn{3}{|c|}{Penalty Critic}\\
         \hline
         Type & - & v-function \cite{hafner2023mastering}\\
         \hline
         \multicolumn{3}{|c|}{See `Safety Critic' for the remaining} \\
         \multicolumn{3}{|c|}{...} \\
         \hline
    \end{tabular}
    \end{center}
\end{table}

\begin{table}[!htb]
    \begin{center}
    {\caption{LAG hyperparameters \cite{ray2019benchmarking}. We use the default hyperparameters provided in \cite{huang2023safe} for the safety gym benchmark \cite{ray2019benchmarking}. }
    \label{tab:lcpo}}
    \begin{tabular}{| m{5cm} | P{3cm} | P{3cm} |}
    \hline
       Name  & Symbol & Value \\
       \hline
       \multicolumn{3}{|c|}{Augmented Lagrangian}\\
    \hline
        Penalty multiplier & $\mu_k$ & $5 \cdot 10^{-9}$ \\
        Initial Lagrange multiplier & $\lambda^k$ & $0.01$ \\
        Penalty power & $\sigma$ & $10^{-6}$ \\
        Cost Value & $C$ & 1.0 \\
        Cost Threshold & $d$ & $1.0$ \\
        \hline
         \multicolumn{3}{|c|}{Penalty Critic}\\
         \hline
         Type & - & v-function \cite{hafner2023mastering}\\
         \hline
         \multicolumn{3}{|c|}{See `Safety Critic' in Table \ref{tab:ambs}} \\
         \multicolumn{3}{|c|}{...} \\
         \hline
    \end{tabular}
    \end{center}
\end{table}

\clearpage
\newpage

\section{The Augmented Lagrangian}
\label{sec:augmentedlagrangian}
First we define the following objective functions,
\begin{eqnarray*}
    J(\pi) & = & \mathbb{E}_{\pi}\left[ \sum^{\infty}_{t=1}  \gamma^{t-1} \cdot r_t \right] \\
    J_C(\pi) & = & \mathbb{E}_{\pi}\left[ \sum^{\infty}_{t=1}  \gamma^{t-1} \cdot c_t \right] \\
\end{eqnarray*}
We now present the following CMDP objective function,
\begin{eqnarray*}
    \max_{\pi} J(\pi) & \text{s.t.} & J_C(\pi) < d
\end{eqnarray*}
where $d$ is the cost threshold. We now present the Augmented Lagrangian with proximal relaxation \cite{wright2006numerical}. First we note that,
\begin{equation*}
    \max_{\pi} \min_{\lambda \geq 0} \big[ J(\pi) - \lambda \left(J_C(\pi) - d\right)  \big] = \max_{\pi} \begin{cases}
        J(\pi) & \text{if $J_C(\pi) < d$}\\
        -\infty & \text{otherwise}
    \end{cases}
\end{equation*}
This is an equivalent form for the CMDP objective, since if $\pi$ is feasible i.e.~$J_C(\pi) < d$, then the maximum value over $\lambda$ is $\lambda=0$. However, if $\pi$ is not feasible then $\lambda$ can be arbitrary large to solve this equation. In particular, this equation is non-smooth when moving between feasible and infeasible policies. Thus, the following relaxation is used,
\begin{equation*}
    \max_{\pi} \min_{\lambda \geq 0} \left[J(\pi) - \lambda \left(J_C(\pi) - d\right) + \frac{1}{\mu_k}(\lambda - \lambda_k)^2 \right]
\end{equation*}
where $\mu_k$ is a non-decreasing penalty multiplier that depends on the gradient step $k$. The new term encourages $\lambda$ to stay close to the previous estimate $\lambda_k$, giving a smooth approximation of the Lagrangian. Taking the derivative w.r.t.~$\lambda$ gives the following update step for $\lambda$,
\begin{equation*}
    \lambda_{k+1} = \begin{cases}
        \lambda_k + \mu_k(J_C(\pi) - d) & \text{if $\lambda_k + \mu_k(J_C(\pi) - d)\geq 0$}\\
        0 & \text{otherwise}
        
                    \end{cases}
\end{equation*}
The penalty multiplier $\mu_k$ is updated as follows in a non-decreasing fashion with each gradient step,
\begin{equation*}
    \mu_{k+1} = \max \{ (\mu_k)^{1+\sigma}, 1 \}
\end{equation*}
for some small fixed penalty power $\sigma$. To optimise the policy $\pi$ we take gradient steps of the following unconstrained objective,
\begin{equation*}
    \tilde J(\pi, \lambda_k, \mu_k) = J(\pi) - \Psi^C(\pi, \lambda_k, \mu_k)
\end{equation*}
where,
\begin{equation*}
    \Psi^C(\pi, \lambda_k, \mu_k) = \begin{cases}
        \lambda_k(J^C(\pi) - d) + \frac{\mu_k}{2}(J^C(\pi) - d)^2 & \text{if $\lambda_k + \mu_k(J_C(\pi) - d)\geq 0$} \\
        -\frac{(\lambda_k)^2}{2 \mu_k} & \text{otherwise}
    \end{cases}
\end{equation*}
In our implementation, the objective functions $J(\pi)$ and $J_C(\pi)$ are estimated with the bootstrapped TD-$\lambda$ returns $R^{\lambda}_t$ and $C^{\lambda}_t$ respectively. Where,

\begin{eqnarray*}
    R^{\lambda}(s_t) & = &\begin{cases}
        \hat r_t + \gamma ((1- \lambda) \hat V(s_{t+1}) + \lambda R^{\lambda}(s_{t+1})) & \text{if $t < H$} \\
        \hat V(s_t) & \text{if $t=H$}
    \end{cases}\\
    C^{\lambda}(s_t) & = &\begin{cases}
        \hat c_t + \gamma ((1- \lambda) \hat V_C(s_{t+1}) + \lambda C^{\lambda}(s_{t+1})) & \text{if $t < H$} \\
        \hat V_C(s_t) & \text{if $t=H$}
    \end{cases}
\end{eqnarray*}
where $\hat r_t$ and $\hat c_t$ are reward and cost estimates resp., and  $\hat V(s_t)$ and $\hat V_C(s_t)$ are value estimates for the return and cost return resp. These estimates are computed from the corresponding components of the world model \cite{hafner2019learning, hafner2023mastering} and critics.
\clearpage
\newpage

\section{Additional Results}
\subsection{Comparison with DreamerV3}
\label{sec:additionalresults}

\begin{figure}[!hbt]
    \centering
    \begin{subfigure}{0.235\textwidth}
         \centering
         \includegraphics[width=\textwidth]{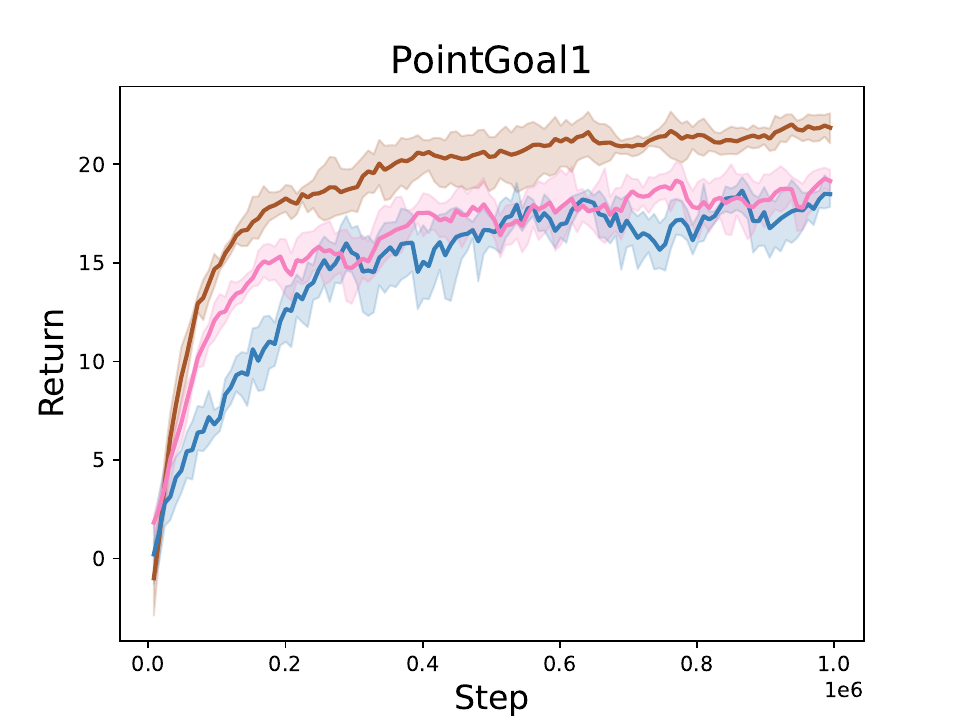}
     \end{subfigure}
     \hfill
     \begin{subfigure}{0.235\textwidth}
         \centering
         \includegraphics[width=\textwidth]{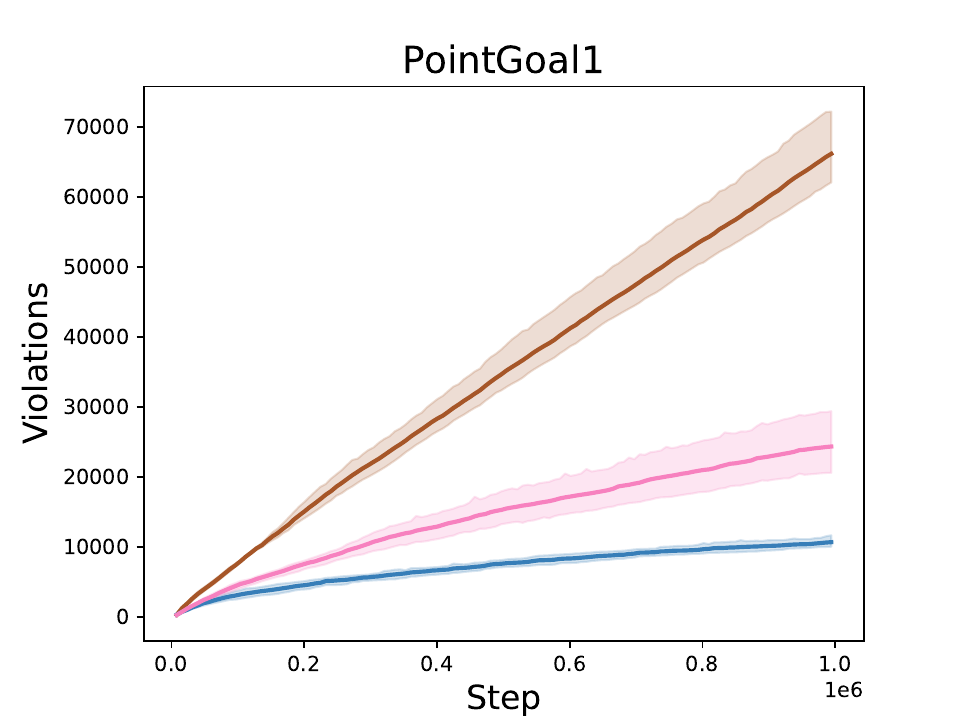}
     \end{subfigure}
     \hfill
     \begin{subfigure}{0.235\textwidth}
         \centering
         \includegraphics[width=\textwidth]{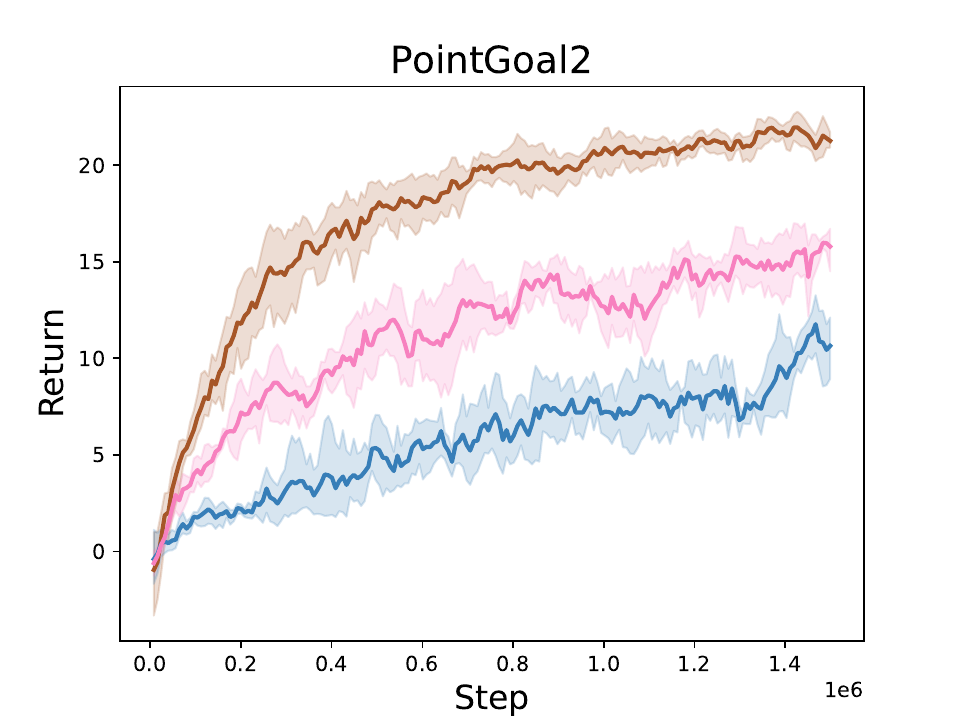}
     \end{subfigure}
     \hfill
     \begin{subfigure}{0.235\textwidth}
         \centering
         \includegraphics[width=\textwidth]{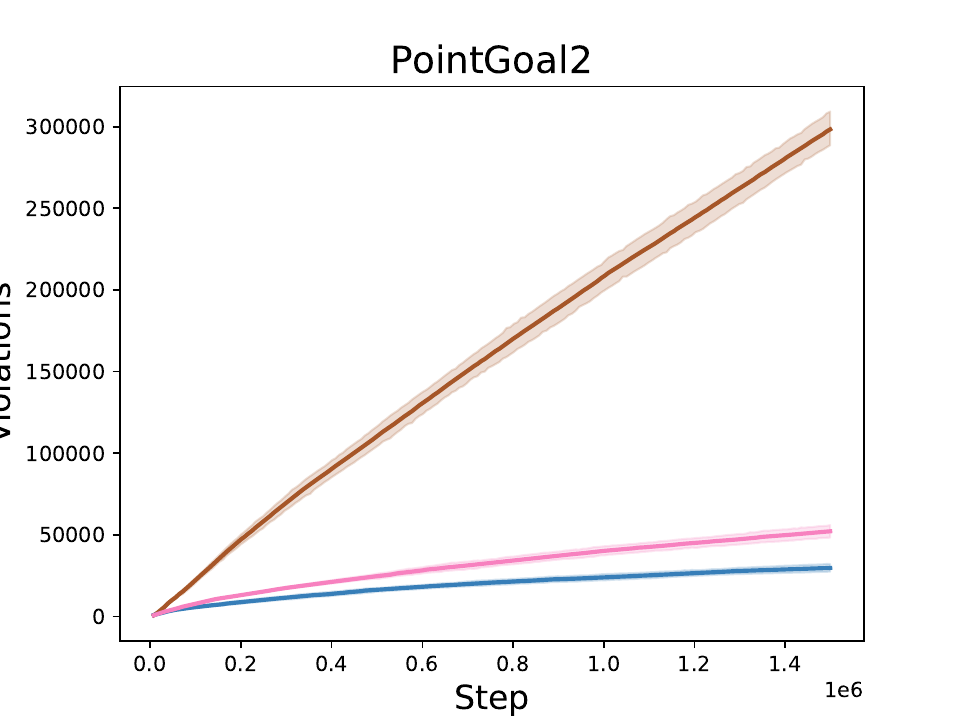}
     \end{subfigure}
     \hfill
     \begin{center}
     \begin{subfigure}{0.235\textwidth}
         \centering
         \includegraphics[width=\textwidth]{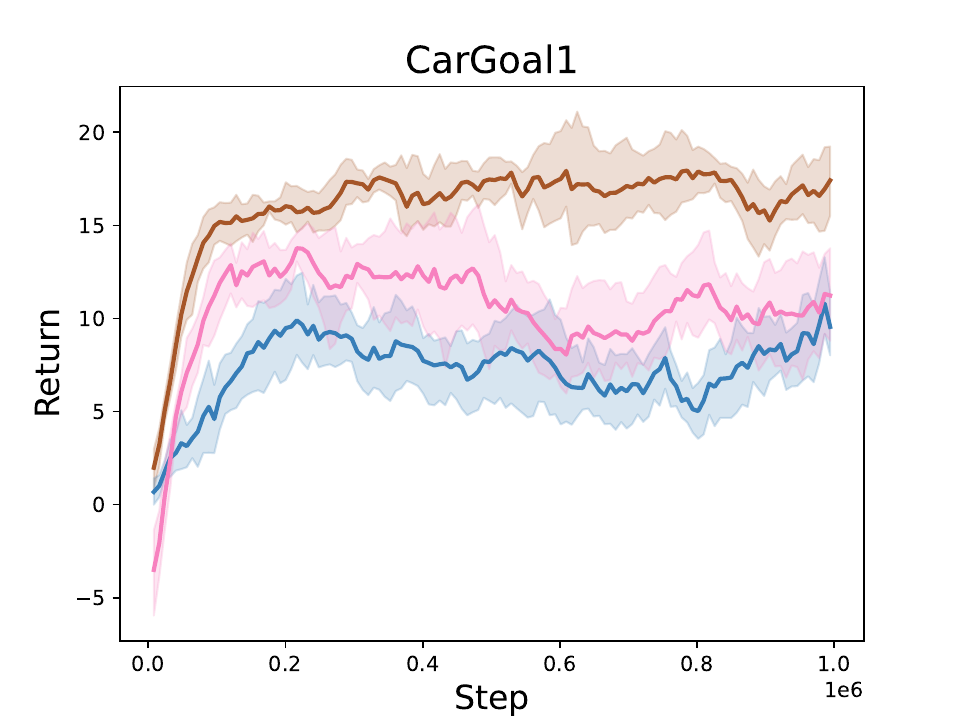}
     \end{subfigure}
     \begin{subfigure}{0.235\textwidth}
         \centering
         \includegraphics[width=\textwidth]{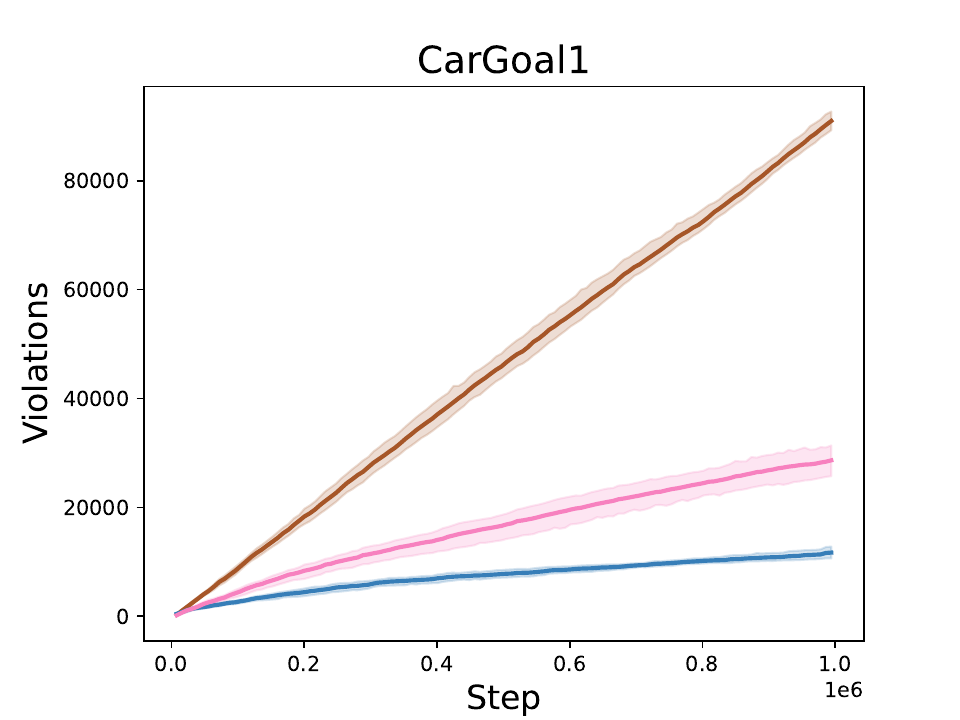}
     \end{subfigure}
     \end{center}
     \hfill
     \newline
     \begin{subfigure}{0.45\textwidth}
         \centering
         \includegraphics[width=\textwidth]{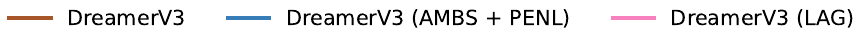}
     \end{subfigure}
    \caption{Comparison between DreamerV3 \cite{hafner2023mastering}, DreamerV3 + LAG and AMBS \cite{goodallb2023approximate} with penalty critic. Episode return (left) and cumulative violations (right) for the three separate algorithms on three Safety Gym environments.}
\end{figure}

\subsection{Comparison with AMBS}
\begin{figure}[!hbt]
    \centering
     \begin{center}
     \begin{subfigure}{0.235\textwidth}
         \centering
         \includegraphics[width=\textwidth]{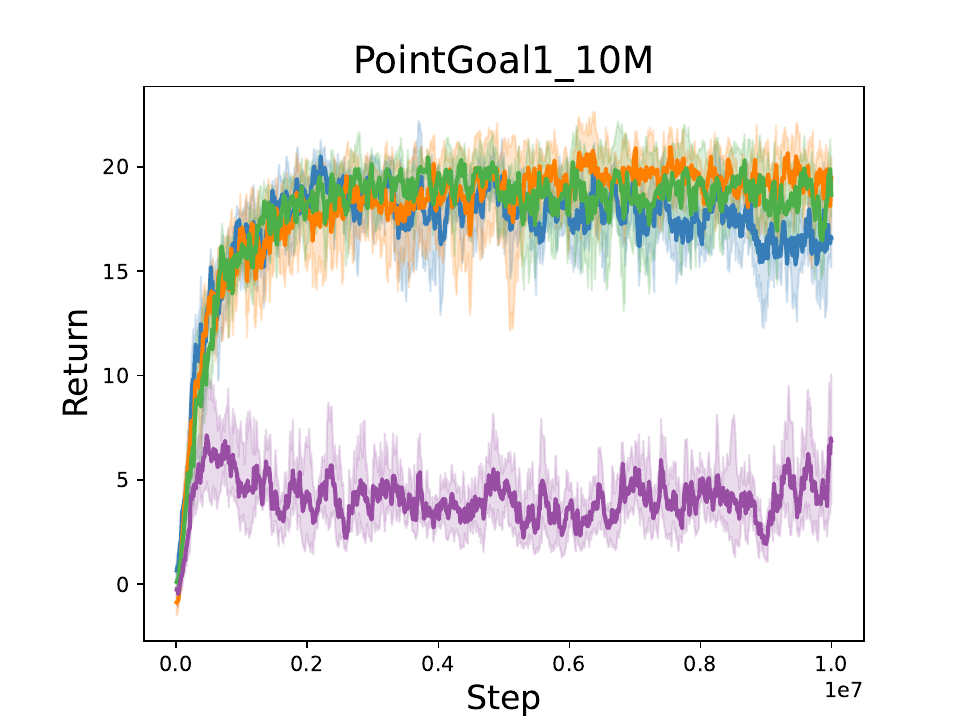}
     \end{subfigure}
     \begin{subfigure}{0.235\textwidth}
         \centering
         \includegraphics[width=\textwidth]{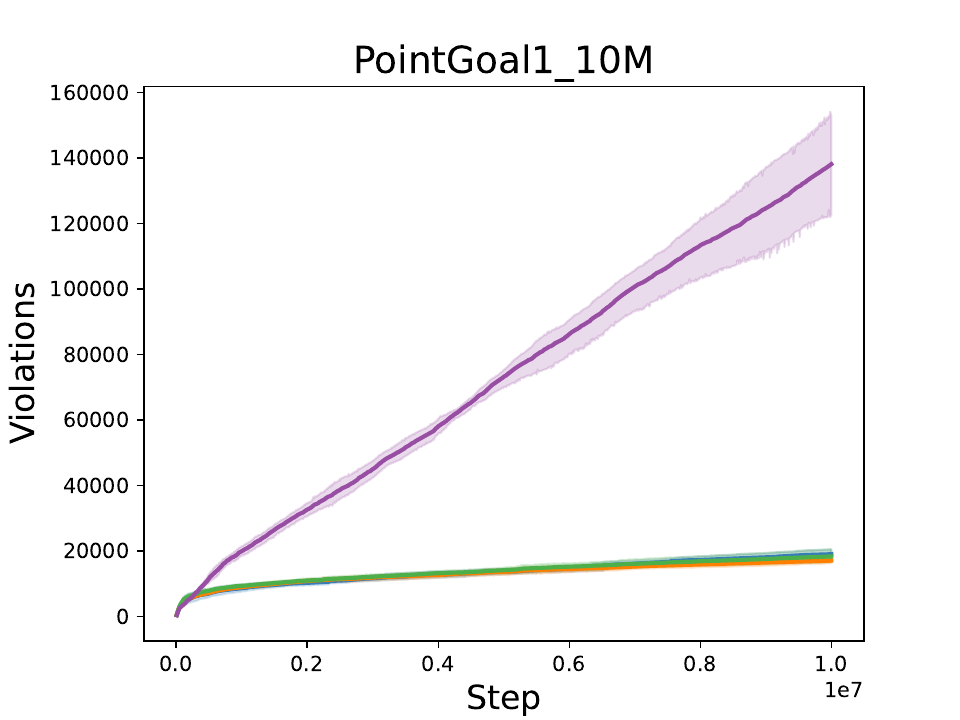}
     \end{subfigure}
     \end{center}
     \hfill
     \newline
     \begin{subfigure}{0.45\textwidth}
         \centering
         \includegraphics[width=\textwidth]{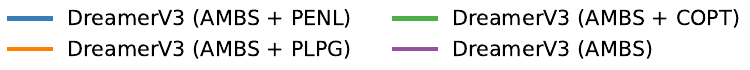}
     \end{subfigure}
    \caption{Comparison between AMBS \cite{goodallb2023approximate} and the three penalty techniques, PENL, COPT and PLPG. Episode return (left) and cumulative violations (right) for the PointGoal1 (10M frames) environment.}
\end{figure}

\clearpage
\newpage

\section{Environment Description}
\label{sec:safetygym}

\subsection{Safety Gym}

Safety Gym \cite{ray2019benchmarking} uses the OpenAI Gym interface \cite{brockman2016openai} for environment interaction. It is built on the advanced physics simulator, MuJoCo \cite{todorov2012mujoco}. In our experiments we use three separate environments: {\em PointGoal1}, {\em PointGoal2} and {\em CarGoal1}. The first two environment use the {\em Point} vehicle for navigation in the environment and the third environment uses the {\em Car} vehicle. 

\subsection{Vehicles}

Fig.~\ref{fig:vehicles} illustrates the two vehicles: {\em Point} and {\em Car}. The neural network controllers for both the vehicles output vectors of values scaled to the range $[-1, +1]$. We provide descriptions for both vehicles below.

\begin{figure}[!htb]
    \centering
    \begin{minipage}{0.3\textwidth}
    \begin{subfigure}{0.4\textwidth}
        \centering
        \includegraphics[width=\textwidth]{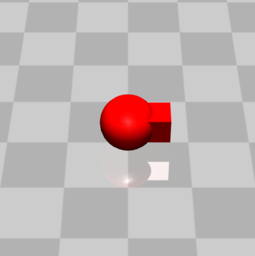}
        \caption{Point}
    \end{subfigure}
    \hfill
    \begin{subfigure}{0.4\textwidth}
        \centering
        \includegraphics[width=\textwidth]{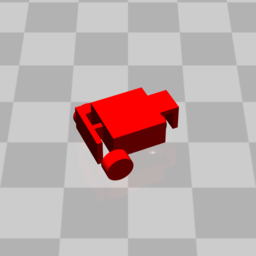}
        \caption{Car}
    \end{subfigure}
    \caption{Safety Gym vehicles.}
    \label{fig:vehicles}
    \end{minipage}
\end{figure}
\noindent \textbf{Point.} A simple vehicle that is constrained to the 2D plane. One actuator controls the rotation of the vehicle the other controls its linear movement. The simple control scheme makes the robot easy to control. The corresponding action space is $[-1, +1]^2$.\\

\noindent \textbf{Car.} This vehicle exhibits more realistic dynamics, with two independently driven wheels. The mechanics are 3D, although the vehicle mostly resides in the 2D plane. The vehicle has a free moving wheel at the front that is not connected to any actuators. The control scheme is a little more advanced since, rotating, and moving forward and backwards require coordination between the two driven wheels. However, the action space is identical, i.e.~$[-1, +1]^2$.

\subsection{Tasks and Constraints}

\textbf{Goal.} In all three environments (i.e.~{\em PointGoal1}, {\em PointGoal2} and {\em CarGoal1}) the goal is to navigate to a series of goal positions, indicated to by a big green circle, see Fig.~\ref{fig:goal}. When the goal is achieved a new goal appears randomly elsewhere in the environment and the agent must first locate (visually) and then navigate to the new goal position. Conveniently the green circle is taller than all the other components of the environment so the agent should be able to determine its location by spinning around just once. This is different to lidar observations, where the agent immediately knows the direction of the goal position, even if it is not within view. A reward of $+1$ is obtained on reaching a goal position. The agent must reach as many goal positions as possible within the episode duration (typically $1000$ frames). \\

\noindent \textbf{Constraints.} In addition to reaching goal positions the agent must avoid a series of obstacles that are randomly placed at the start of each episode. For the {\em PointGoal1 } and {\em CarGoal1} environments the agent must only avoid hazardous areas indicated by blue circles, see Fig.~\ref{fig:hazards}. The {\em PointGoal2} environment is harder. Firstly, the environment is bigger, meaning it will take longer to navigate across the environment and there are more obstacles to avoid. In addition, the agent must not only avoid hazardous areas but must also avoid collisions with vases, see Fig.~\ref{fig:vases}. We summarise the environments and their corresponding safety-formula below:
\begin{itemize}
    \item {\em PointGoal1 }: $\Psi = \neg hazard$.
    \item {\em PointGoal2 }: $\Psi = \neg hazard \land \neg collision$.
    \item {\em CarGoal1 }: $\Psi = \neg hazard$
\end{itemize}

\begin{figure}[!h]
    \centering
    \begin{minipage}{0.45\textwidth}
    \begin{subfigure}{0.3\textwidth}
        \centering
        \includegraphics[width=\textwidth]{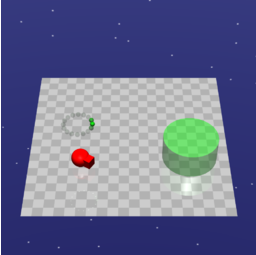}
        \caption{Goal positions}
        \label{fig:goal}
    \end{subfigure}
    \hfill
    \begin{subfigure}{0.3\textwidth}
        \centering
        \includegraphics[width=\textwidth]{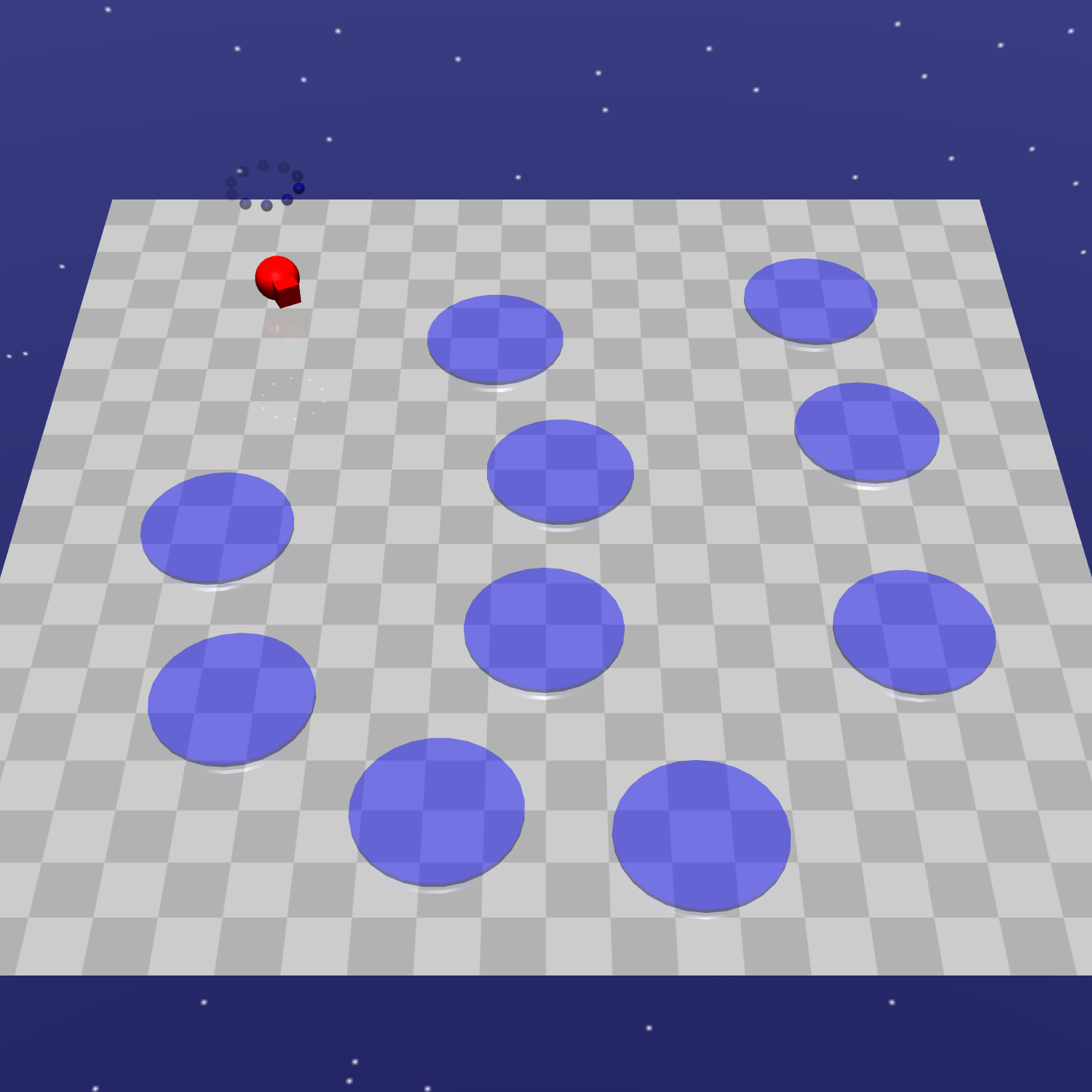}
        \caption{Hazardous areas}
        \label{fig:hazards}
    \end{subfigure}
    \hfill
    \begin{subfigure}{0.3\textwidth}
        \centering
        \includegraphics[width=\textwidth]{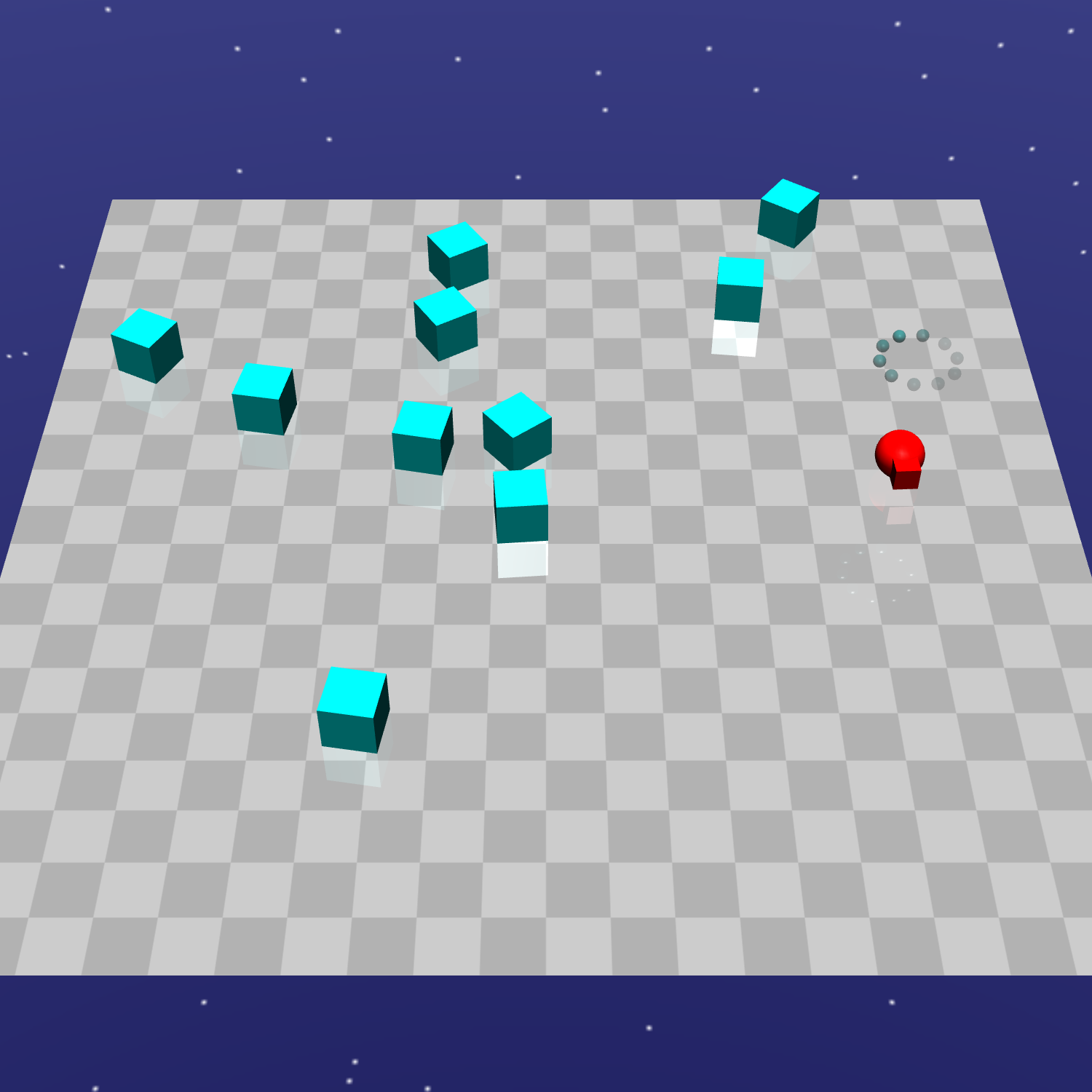}
        \caption{Vases}
        \label{fig:vases}
    \end{subfigure}
    \caption{Safety Gym tasks and constraints.}
    \label{fig:tasksandconstraints}
    \end{minipage}
\end{figure}
\end{document}